\documentclass[10pt,a4paper,english]{article}
\pdfoutput=1
\makeatletter
\usepackage{babel}
\makeatother 

\usepackage{amsmath,latexsym,amssymb,amssymb,amsthm}  
              
\usepackage{graphicx,epsfig,setspace,graphics,color}

\usepackage{caption}

\usepackage{cite}

\usepackage[colorlinks=false,
        linkcolor=black,
        citecolor=black,
        filecolor=black,
        urlcolor=black,
        plainpages=false,
        pdfpagelabels=true]{hyperref}
        
\usepackage{subfig}

\newcount\Comments  
\usepackage{color}
\definecolor{darkgreen}{rgb}{0,0.5,0}
\definecolor{purple}{rgb}{1,0,1}
\definecolor{darkyellow}{rgb}{0.5,0.5,0}
\newcommand{\kibitz}[2]{\ifnum\Comments=1\textcolor{#1}{#2}\fi}

\newcommand{\N}{\ensuremath{\mathbb{N}}}
\newcommand{\R}{\ensuremath{\mathbb{R}}}

\newcommand{\PP}{\ensuremath{\mathbb{P}}}
\newcommand{\bc}{\ensuremath{\boldsymbol{c}}}
\newcommand{\bt}{\ensuremath{\boldsymbol{t}}}
\newcommand{\bx}{\ensuremath{\boldsymbol{x}}}
\newcommand{\by}{\ensuremath{\boldsymbol{y}}}
\newcommand{\bz}{\ensuremath{\boldsymbol{z}}}
\newcommand{\bxi}{\ensuremath{\boldsymbol{\xi}}}

\newcommand{\argmin}[1]{\underset{#1}{\operatorname{arg}\,\operatorname{min}}\;}

\newcommand{\outerRKHS}{\ensuremath{H\left(\Phi,\R\right)}}
\newcommand{\innerRKHS}{\ensuremath{\boldsymbol{H}\left(\Omega,\Phi\right)}}

\newtheorem{theorem}{Theorem}  
\newtheorem{lemma}[theorem]{Lemma}

\newtheorem{corollary}[theorem]{Corollary}  

\newtheorem{definition}[theorem]{Definition}

\begin{document}
\author{Bastian Bohn\footnotemark[2] \and Michael Griebel\footnotemark[2] \footnotemark[3] \and Christian Rieger\footnotemark[2]}
\date{\today}
\title{A representer theorem for deep kernel learning}

\maketitle

\renewcommand{\thefootnote}{\fnsymbol{footnote}}

\footnotetext[2]{Institute for Numerical Simulation, University of Bonn, Wegelerstr. 6, 53115 Bonn, Germany.}
\footnotetext[3]{Fraunhofer Institute for Algorithms and Scientific Computing SCAI, Schloss Birlinghoven, 53754 Sankt\\ Augustin, Germany.\\
The authors want to thank the anonymous referees for their suggestions and remarks and especially for pointing out the relation to \cite{DinuzzoPhd}.
The authors were partially supported by the Sonderforschungsbereich 1060 \textit{The Mathematics of Emergent Effects} funded by the Deutsche Forschungsgemeinschaft.}

\renewcommand{\thefootnote}{\arabic{footnote}}

\begin{abstract}%
In this paper we provide a finite-sample and an infinite-sample representer theorem for the concatenation of 
(linear combinations of) kernel functions of reproducing kernel Hilbert spaces. 
These results serve as mathematical foundation for the analysis of machine learning algorithms based on compositions of functions. 
As a direct consequence in the finite-sample case, the corresponding infinite-dimensional minimization problems can be recast into (nonlinear) finite-dimensional 
minimization problems, which can be tackled with nonlinear optimization algorithms.
Moreover, we show how concatenated machine learning problems can be reformulated as neural networks
and how our representer theorem applies to a broad class of state-of-the-art deep learning methods.
\end{abstract}


\section{Introduction}
\label{sec:Introduction}

The interpolation or regression of given function values is one of the main tasks in modern data mining and machine learning applications. Due to the famous representer theorem 
for empirical risk minimization in reproducing kernel Hilbert spaces (RKHS), see e.g.~\cite{Smola,Steinwart,WahbaKimeldorf}, various algorithms based on finite linear combinations of kernel translates have gained much popularity in the last decade, like, for example, support vector machines (SVMs) and 
Tikhonov-regularized least-squares in RKHS. 
In general, these methods work very well if the underlying problem fits the chosen reproducing kernel space $H$, e.g.~if the 
given input values stem from a function $g \in H$.
However, if $H$ contains for instance only smooth functions but $g$ has a kink or a jump, the interpolant or regressor, respectively, 
in $H$ might not represent a good approximation to the true function $g$ anymore. 
Then, if it is not known how to choose an appropriate kernel $K$ of $H$ a priorily, one usually relies on so-called multiple kernel learning (MKL) algorithms, which try 
to determine the optimal kernel adaptively, see e.g.~\cite{BachMKL}. But while most of these methods allow to learn a suitable kernel by simply constructing a linear or convex combination
of a given set of input kernels, they still do not achieve considerably better results than standard a-priori kernel choices for many applications, see \cite{Goenen}. 

In recent years, promising new variants of kernel learning methods, namely deep kernel learning and multi-layer-MKL (MLMKL)
algorithms have been developed. They have proven to be very successful in regression and classification tasks.
Here, motivated by multi-layer feed-forward neural networks, 
a kernel function is concatenated with one or more nonlinear functions in order to achieve a highly flexible new kernel function, 
see e.g.~\cite{ChoSaul1,LawrenceDeepGauss,Rebai,Strobl,Wilson,Zhuang}. 
The main idea behind this approach is to combine the flexibility of deep neural networks, in which the feature detection in the data set is done completely automatically,
with the approximation power of kernel methods, in which a feature map is determined by the chosen kernel. This way, the neural network architecture learns the optimal 
kernel that best represents important features of the data for the task at hand.
 While first steps towards creating a mathematical framework to analyze 
deep neural networks - especially for image classification tasks - have been made in e.g.~\cite{Mallat,Mhaskar,DeepTaylorDecomp}, deep approximation theory 
for kernel based approaches is still missing at large. Moreover, the underlying nonlinear minimization problem is usually tackled by simple gradient descent and 
heuristic backpropagation algorithms without a thorough theoretical analysis of its properties. An initial cornerstone for the analysis of chained kernel approximations has been 
provided by \cite{DinuzzoPhd}, where two-layer kernel networks were considered and their relation to MKL was established. However, an analysis of deeper kernel networks and their connection 
to MLMKL has not been considered so far.

In this paper, we consider the problem of optimal concatenated approximation in reproducing kernel Hilbert spaces, 
which will directly lead to a variant of multi-layer kernel learning problems and will extend the results achieved in \cite{DinuzzoPhd}. For this class, we will prove a representer theorem, 
which allows us to reduce the nonlinear, potentially infinite-dimensional optimization problem to a finite-dimensional one. Consequently, standard nonlinear optimization techniques can 
be used to tackle this problem. At least to our knowledge, this is the first derivation of a representer theorem for concatenated function approximation in the literature. 
It is also valid for certain types of hidden layer neural networks and deep SVMs. 

The remainder of this paper is organized as follows: 
In Section \ref{sec:IntRKHS}, we briefly review the interpolation and the regression problem in an (possibly infinite-dimensional) RKHS
and discuss how the classical representer theorem allows to recast these problems into finite-dimensional linear equation systems. 
In Section \ref{sec:IntChained}, we introduce the optimal {\em concatenated} approximation problem for arbitrary loss functions and regularizers. We 
derive a representer theorem for this problem in the multi-layer case and discuss its relation to deep learning and multi-layer multiple kernel learning methods.
Furthermore, we exemplarily derive algorithms for interpolation and least-squares regression in the two-layer case from it. The latter will be a natural generalization of 
the RLS2 method developed in \cite{DinuzzoPhd}, which only deals with a linear outer kernel.
Section \ref{sec:Experiments} illustrates the application of our concatenated interpolation and regression algorithms to two simple examples and serves as a proof of concept.
Finally, we conclude with a summary and an outlook 
in Section \ref{sec:Conclusion}.

\section{Interpolation and regression in reproducing kernel Hilbert spaces}
\label{sec:IntRKHS}
In this section we shortly review interpolation and least-squares regression problems, respectively, in an RKHS. 
To this end, we consider the standard representer theorem and show how it helps to find an interpolant/regressor.

\subsection{Interpolation}

Let $\Omega \subset \R^d$ be an open domain and let the pairwise disjoint points $X:=\left\{\bx_{1},\ldots,\bx_{N} \right\}\subset \Omega$ 
and the values $Y := \{y_1,\ldots,y_N\} \subset \R$ be given. Let furthermore $H := H(\Omega,\R)$ be a reproducing kernel Hilbert space of real-valued functions on $\Omega$.
The minimal norm interpolant is
\begin{equation} \label{eq:interpolStd}
 f^*_{X,Y} := \argmin{f \in H} \| f \|_{H} ~~ \text{ such that } ~ f(\bx_i) = y_i ~~ \forall\, i = 1,\ldots,N.
\end{equation}
The classical representer theorem, see e.g.~\cite{Smola,Steinwart} for scalar-valued functions and \cite{Micchelli} for vector-valued functions,
now states that $f^*_{X,Y}$ can be written as a {\em finite} linear combination of kernel evaluations in the data, namely
\begin{equation} \label{eq:RepresenterTheorem}
 f^*_{X,Y}(\bx) = \sum_{i=1}^N \alpha^*_i K(\bx_i, \bx),
\end{equation}
where $K: \Omega \times \Omega \to \R$ denotes the reproducing kernel of $H$ and $\alpha^*_i \in \R, i=1,\ldots,N$, are the corresponding coefficients. For details on 
RKHS, see \cite{Aron}. Therefore, the 
solution to the possibly infinite-dimensional optimization problem \eqref{eq:interpolStd} resides in the $N$-dimensional span of the functions $K(\bx_i, \cdot), i=1,\ldots,N$. 
To compute the coefficients, we simply have to solve the system 
\begin{equation} \label{eq:LGSRepresenterTheoremInt}
\boldsymbol{M}_{X,X} \boldsymbol{\alpha}^* = \boldsymbol{y}
\end{equation} 
of linear equations with 
\begin{equation} \label{eq:MatrixCoeffsAndRHS}
\boldsymbol{M}_{X,X} := \begin{pmatrix} 
K(\bx_{1},\bx_{1}) & \ldots & K(\bx_{1},\bx_{N}) \\
\vdots & \ddots & \vdots\\
K(\bx_{N},\bx_{1}) & \ldots & K(\bx_{N},\bx_{N}) 
\end{pmatrix},
\quad 
\boldsymbol{\alpha}^* :=
\begin{pmatrix}
\alpha^*_{1} \\ \vdots \\ \alpha^*_{N}
\end{pmatrix} 
\quad \text{and } \boldsymbol{y} :=\begin{pmatrix}
y_{1} \\ \vdots \\ y_{N}
\end{pmatrix}.
\end{equation}
Note that this $N \times N$ system admits a unique solution if the kernel $K$ is strictly positive definite. For example, for Sobolev kernels it can be shown
that the condition number of the system matrix $\boldsymbol{M}_{X,X}$ only grows moderately with the size $N$ provided that
the data points are quasi-uniformly distributed, see \cite{SchabackdeMarchi}. 
Moreover, for infinitely smooth kernel functions (e.g.~Gaussian kernels or multiquadrics) it can be necessary to perform an appropriate basis change 
before solving the above equation system, see e.g.~\cite{Wendland}.

\subsection{Least-squares regression}

In real-world applications, the values $y_i, i=1,\ldots,N$ are usually not exactly given, but are perturbed by some noise term. 
Therefore, a direct interpolation might no longer be appropriate. In this case, one considers the corresponding regularized least-squares regression problem
\begin{equation} \label{eq:regressStd}
 f^{\lambda}_{X,Y} := \argmin{f \in H} \lambda \| f \|^2_{H} + \sum_{j=1}^N |f(\bx_i) - y_i|^2,
\end{equation}
where the side condition in \eqref{eq:interpolStd} is substituted by a penalty term. Here, the Lagrange multiplier $\lambda$ weights
the importance of the norm minimization against the function evaluation error. 
Again, the representer theorem \cite{Micchelli, Smola, Steinwart} tells us that $f^{\lambda}_{X,Y}$ is of the form \eqref{eq:RepresenterTheorem}, i.e.~
$$
 f^{\lambda}_{X,Y}(\bx) = \sum_{i=1}^N \alpha^{\lambda}_i K(\bx_i, \bx).
$$
This time the coefficients $\alpha^{\lambda}_i,i=1,\ldots,N$, are determined by
\begin{equation} \label{eq:LGSRepresenterTheoremReg}
 \left(\boldsymbol{M}_{X,X} + \lambda \boldsymbol{I} \right) \boldsymbol{\alpha}^{\lambda} = \boldsymbol{y},
\end{equation}
where $\boldsymbol{I}$ denotes the $N \times N$ identity matrix.
The size of the Lagrange parameter $\lambda > 0$ now also influences the condition number of the system matrix, i.e.~the larger $\lambda$ is, the smaller the condition number becomes.

\section{Interpolation and regression with compositions of reproducing kernel Hilbert spaces}
\label{sec:IntChained}

As already mentioned in the introduction, the standard interpolation and regression algorithms in RKHS work well if the samples 
$y_i$ are (perturbed) evaluations of a function $g \in H$, where the reproducing kernel space $H$ is known in the first place.
However, if the appropriate RKHS $H$ is unknown, it is advisable to resort to multiple kernel learning methods or multi-layer multiple kernel learning methods.

We now explain this aspect in more detail and, to this end, motivate a first two-dimensional, two-layer approach with an example: 
Let the kernel $K$ of $H$ be a tensor-product of
two univariate Mat\'{e}rn Sobolev kernels of order one on $\R$, see Section \ref{sec:Experiments} for a definition of this kernel. The corresponding function space $H$ is often also
called Sobolev space of ``mixed smoothness'' of order one and it is of special importance for e.g.~sparse grid discretizations, see \cite{BungartzGriebel}, and quasi Monte Carlo quadrature, see \cite{Hinrichs.Markhasin.Oettershagen}.
Now, let us consider the continuous function $g_1(x,y) := (0.1 + |x|)^{-1}$, which has a kink that is perpendicular to the $x$-axis. It can easily be shown that $g_1 \in H$ 
and, therefore, the interpolant of $g_1$ by a function from $H$ resembles a good approximation to $g_1$, see Figure \ref{fig:H1mixRotExample}(a). 
If we now look at $g_2(x,y) :=  (0.1+|x-y|)^{-1}$, which has a kink along the diagonal with $x=y$, then $g_2 \notin H$. 
Therefore, the interpolant of $g_2$ by a function in $H$ is a rather bad approximation to $g_2$. This can be seen in Figure \ref{fig:H1mixRotExample}(b).
However, if we let $R^{-1}$ be
a rotation by $45^{\circ}$, then $g_2 \circ R^{-1} \in H$ would have an axis-aligned kink like $g_1$.
To use this fact when interpolating $g_2$, we can simply look for the best 
interpolant in $\{ f \circ R \mid f \in H\}$ in \eqref{eq:interpolStd} instead of $f \in H$. This example is illustrated in Figure \ref{fig:H1mixRotExample}(c). 
As we can see, the interpolant in Figure \ref{fig:H1mixRotExample}(c) is a much better representative for $g_2$ than the one in Figure \ref{fig:H1mixRotExample}(b).
This example illustrates that, already in the very simple case of employing a concatenation with a rotation, a two-layer approach can be a good choice to overcome the restrictions 
of a standard kernel learning algorithm. 
Let us remark that already this motivating example exhibits a fundamentally different setting from the one considered in \cite{DinuzzoPhd} because of the 
nonlinearity of the outer kernel. While the RLS2 algorithm introduced there can be interpreted as an MKL variant, where a convex combination of 
given kernel functions is computed, we are looking for an inner function, which transforms the domain in such a way that it is optimal for the (possibly nonlinear) outer kernel.

\begin{figure}
\centering
\subfloat[$y_i = g_1(\bx_i), f \in H$]{
\includegraphics[width=0.32\linewidth]{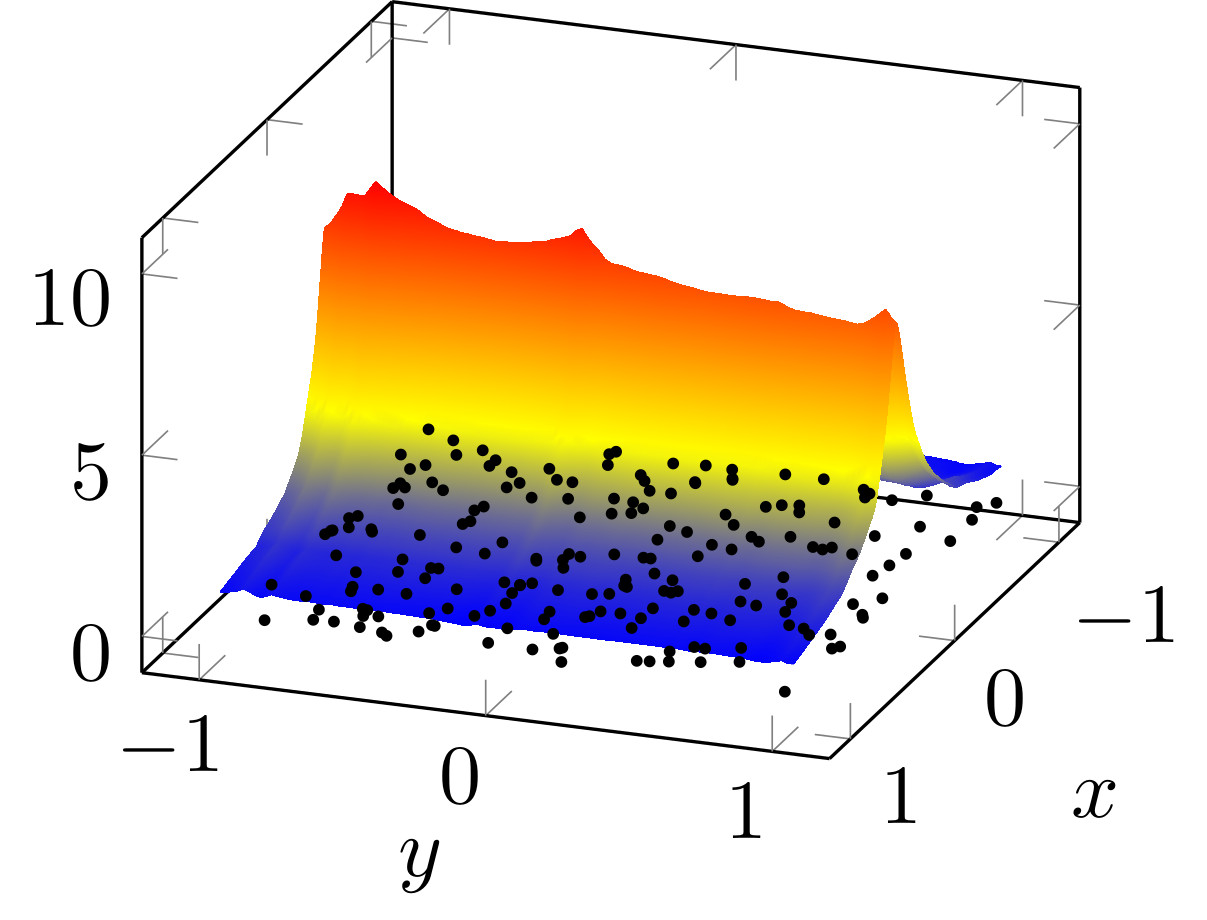}}
\hfill
\subfloat[$y_i = g_2(\bx_i), f \in H$]{
\includegraphics[width=0.32\linewidth]{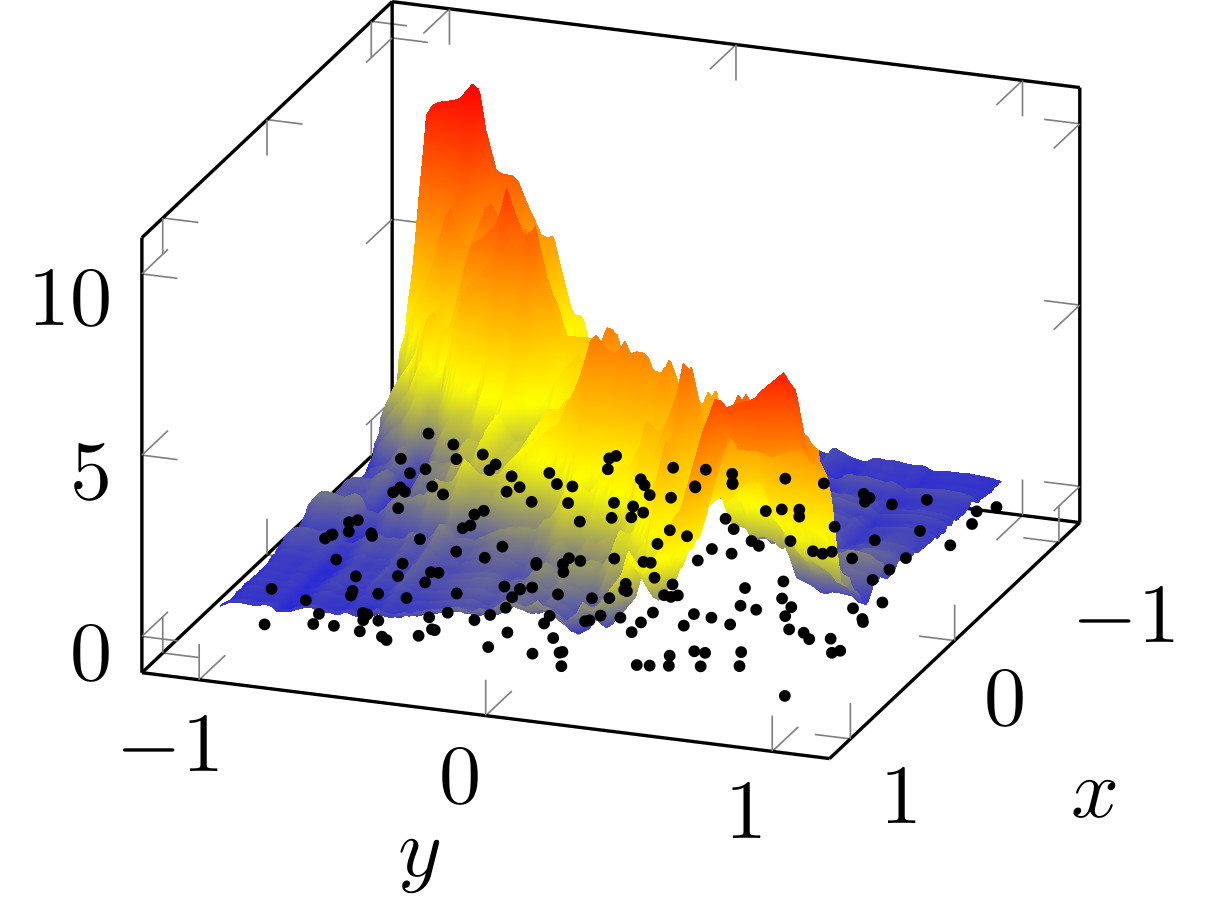}}
\hfill
\subfloat[$y_i = g_2(\bx_i), f \in \{ h \circ R \mid h \in H \}$]{
\includegraphics[width=0.32\linewidth]{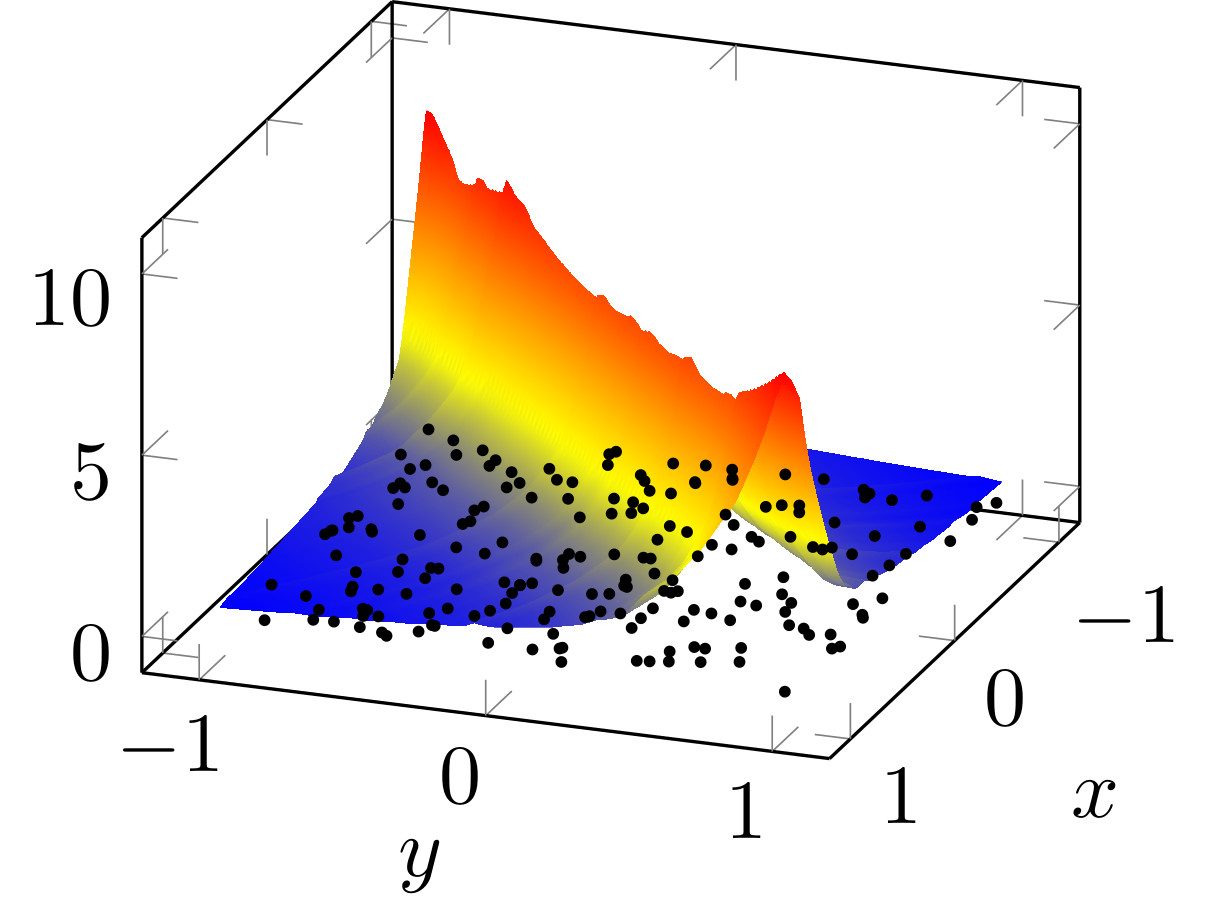}}
\caption{Solutions to \eqref{eq:interpolStd} in the two-variate tensor-product Mat\'{e}rn-kernel Sobolev space $H$ of order one, see also \cite{FasshauerYe}, with 
$200$ uniform samples $\bx_i, i=1,\ldots,200$ (marked in black), shown in the domain $[-1,1]^2$. (a) depicts the solution $f \in H$ for values $y_i$ sampled from $g_1$, whereas (b) shows the optimal 
solution for $y_i$ sampled from $g_2$. (c) presents the best interpolant of 
type $f \circ R$, where $f \in H$ and $R$ is a rotation by $45^{\circ}$ for $y_i$ sampled from $g_2$. 
For reasons of comparability, we restricted our representation to $[-1,1]^2$ here, although
some data points were mapped outside of this domain by applying the rotation $R$ and the kernel was defined on the whole $\R^2$.}
\label{fig:H1mixRotExample}
\end{figure}

Now, instead of just considering one layer of simple rotations as in the above example, we allow for a fully flexible multi-layer kernel learning approach, 
where we employ arbitrary functions from reproducing kernel Hilbert spaces in each layer.
This approach can successfully deal with a much broader class of interpolation and regression problems, see also \cite{Rebai,Zhuang}. 
To this end, we consider \emph{concatenated} machine learning problems. We introduce a new representer theorem for the case 
of multiple concatenations of functions from RKHS, which allows us to derive 
the related, finite-dimensional, nonlinear optimization problem.

\subsection{A representer theorem for concatenated kernel learning}
\label{subsec:generalRepTheo}

In this section, we show how a {\em concatenated representer theorem} can be derived for a very general class of problem types and an arbitrary number $L \in \N$ of concatenations. 
For more details on vector-valued reproducing kernel Hilbert spaces, we refer the reader to \cite{Micchelli}. For a two-layer variant of this theorem, we refer to \cite{DinuzzoPhd}.

\begin{theorem} \label{theo:generalRepTheo}
 Let ${\cal H}_1, \ldots, {\cal H}_L$ be reproducing kernel Hilbert spaces of functions with finite-dimensional domains $D_l$ and ranges $R_l \subseteq \R^{d_l}$ with $d_l \in \N$ for
 $l=1,\ldots,L$ 
 such that $R_l \subseteq D_{l-1}$ for $l=2,\ldots,L$, $D_L = \Omega$ and $R_1 \subseteq \R$. Let furthermore
 ${\cal L} : \R^2 \to [0,\infty]$ be an arbitrary 
 loss function and let $\Theta_1, \ldots, \Theta_L : [0,\infty) \to [0,\infty)$ be strictly monotonically increasing functions. 
 Then, a set of minimizers $\left(f_l\right)_{l=1}^L$ with $f_l \in {\cal H}_l$ of 
 \begin{equation} \label{eq:generalMinFunct}
 J(f_1,\ldots,f_L) := \sum_{i=1}^N {\cal L}\left(y_i,f_1 \circ \ldots \circ f_L(\bx_i)\right) + \sum_{l=1}^L \Theta_l\left(\| f_l \|^2_{{\cal H}_l}\right)
 \end{equation}
fulfills $f_l \in \tilde{V}_l \subset {\cal H}_l$ for all $l=1,\ldots,L$ with 
$$\tilde{V}_l = \operatorname{span}\left\{ K_l\left( f_{l+1} \circ \ldots \circ f_{L}\left(\bx_i\right), \cdot \right) \boldsymbol{e}_{k_l} \mid 
i = 1,\ldots,N \text{ and } k_l = 1,\ldots,d_l\right\},$$
where $K_l$ denotes the reproducing kernel of ${\cal H}_l$ and $\boldsymbol{e}_{k_l} \in \R^{d_l}$ is the $k_l$-th unit vector.
\end{theorem}

\begin{proof}
We denote by $\Pi_{\tilde{V}_l}$ and $\Pi_{\tilde{V}_l}^{\perp}$ the projector 
 onto $\tilde{V}_l$ and its orthogonal complement in ${\cal H}_l$, respectively, for $l = 1,\ldots,L$. First, we note that
 \begin{align*}
 f_l \circ f_{l+1} \circ \ldots \circ f_L(\bx_i)  & = 
\sum_{k=1}^{d_{l}}\left(\Pi_{ \tilde{V}_l }\left(f_l\right) + \Pi_{ \tilde{V}_l^{\perp} }\left(f_l \right) ,
K_l\left(f_{l+1} \circ \ldots \circ f_L(\bx_i),\cdot \right)\boldsymbol{e}_{k} \right)_{{\cal H}_l} \cdot \boldsymbol{e}_{k} \\
& = \sum_{k=1}^{d_{l}} \left(\Pi_{ \tilde{V}_{l} }\left(f_l\right), K_l\left(f_{l+1} \circ \ldots \circ f_L(\bx_{i}),\cdot \right)\boldsymbol{e}_{k} \right)_{{\cal H}_l}
\cdot \boldsymbol{e}_{k} \\
&= \sum_{k=1}^{d_{l}} \left(\boldsymbol{e}_{k}^T\Pi_{ \tilde{V}_l }\left(f_l\right) \left(f_{l+1} \circ \ldots \circ f_L(\bx_{i}) \right)\right) \cdot \boldsymbol{e}_{k} 
\\
 &=\Pi_{ \tilde{V}_{l} }\left(f_l\right) \left(f_{l+1} \circ \ldots \circ f_L(\bx_{i}) \right)
 \end{align*}
 for all $i=1,\ldots,N$ and $l=1,\ldots,L$. 
 Since this holds for each function in the chain, we can iterate this process to obtain
 \begin{equation} \label{eq:PointEvalRep}
 f_l \circ f_{l+1} \circ \ldots \circ f_L(\bx_i) = \Pi_{ \tilde{V}_{l} }\left(f_l\right) \circ \Pi_{\tilde{V}_{l+1}}\left(f_{l+1}\right) \circ \ldots \circ 
 \Pi_{\tilde{V}_L} \left(f_L\right) (\bx_{i})
 \end{equation}
 for each $l=1,\ldots,L$. Therefore, we have
 \begin{align*}
 J(f_1,\ldots,f_L) & = \sum_{i=1}^N {\cal L}\left(y_i,\Pi_{\tilde{V}_1}(f_1) \circ \ldots \circ \Pi_{\tilde{V}_L}(f_L)(\bx_i)\right) 
 \\ & + \sum_{l=1}^{L} \Theta_l\left(\| \Pi_{\tilde{V}_l} (f_l) \|^2_{{\cal H}_l} + \| \Pi_{\tilde{V}_l^{\perp}}(f_l) \|^2_{{\cal H}_l} \right) 
 \geq J(\Pi_{\tilde{V}_1}(f_1),\ldots,\Pi_{\tilde{V}_L}(f_L))
 \end{align*}
 and equality only holds if $f_l \in \tilde{V}_l$ for each $l=1,\ldots,L$ because of the strict monotonicity of each $\Theta_l$. This completes the proof.
 \end{proof}
 
 Note that, because of \eqref{eq:PointEvalRep}, we could even state a more general version of Theorem \ref{theo:generalRepTheo} where the loss function ${\cal L}$ not only 
 depends on the point evaluations $f_1 \circ \ldots \circ f_L(\bx_i)$ for $i=1,\ldots,N$, but also on the intermediate values $f_l \circ \ldots \circ f_L(\bx_i)$ for any 
 $l = 2,\ldots,L$. However, for the sake of readability, we proceed with \eqref{eq:generalMinFunct}. 
Theorem \ref{theo:generalRepTheo} now states that 
\begin{equation} \label{eq:generalRepTheo}
(f_1,\ldots,f_L) = \argmin{\underset{l=1,\ldots,L}{\bar{f}_l \in {\cal H}_l}} J(\bar{f}_1,\ldots,\bar{f}_L) = \argmin{\underset{l=1,\ldots,L}{\bar{f}_l \in \tilde{V}_l}} J(\bar{f}_1,\ldots,\bar{f}_L)
\end{equation}
with $J$ from \eqref{eq:generalMinFunct}. This means that the (possibly) infinite-dimensional optimization problem 
$$ \argmin{\underset{l=1,\ldots,L}{\bar{f}_l \in {\cal H}_l}} J(\bar{f}_1,\ldots,\bar{f}_L)$$
can be recast into the finite-dimensional optimization problem 
$$
\argmin{\underset{l=1,\ldots,L}{\bar{f}_l \in \tilde{V}_l}} J(\bar{f}_1,\ldots,\bar{f}_L).
$$
In this way, our representer theorem is a direct extension 
of the classical representer theorem, see Section \ref{sec:IntRKHS} and \cite{Smola}, to concatenated functions.
We obtain that the solution to \eqref{eq:generalRepTheo} is given by a linear combination of at most $N$ basis functions in each layer. Therefore, 
the overall number of degrees of freedom in the underlying optimization problem \eqref{eq:generalRepTheo} is given by 
$$
\# \text{dof} = \sum_{l=1}^L \dim\left(\tilde{V}_l\right) = \sum_{l=1}^L N \cdot d_l = N \cdot \left(1 + \sum_{l=2}^L d_l\right).
$$
According to Theorem \ref{theo:generalRepTheo}, we can write $f_1$ as 
$$
f_1(\cdot) = \sum_{j=1}^N \alpha_j K_1\left(f_2 \circ \ldots \circ f_L(\bx_j), \cdot \right)
$$
for some coefficients $\alpha_j \in \R$. Therefore, the concatenated function $h(\cdot) = f_1 \circ \ldots \circ f_L(\cdot)$, which we are interested in, 
can be expressed as 
$$
h(\cdot) = \sum_{j=1}^N \alpha_j \mathcal{K}^L(\bx_j,\cdot)
$$
with the deep kernel
\begin{equation} \label{eq:deepKernelGeneral}
\mathcal{K}^L(\bx,\by) = K_1\left(f_2 \circ \ldots \circ f_L(\bx), f_2 \circ \ldots \circ f_L(\by)\right).
\end{equation}
Due to the definition of $\tilde{V}_l$ for $l=1,\ldots,L$, the corresponding $f_l$ is defined recursively. In general, it is thus not possible to simply write 
down a closed formula for $\mathcal{K}^L$ for arbitrary $L$. To illustrate the structure of the kernel $\mathcal{K}^L$, we therefore consider 
a two-layer example with $L=2$ in the following.
In this case, we obtain $\tilde{V}_2 = \operatorname{span}\{ K_2(\bx_i, \cdot) \boldsymbol{e}_{k_2} \mid  i=1,\ldots,N \text{ and } k_2 = 1,\ldots,d_2\}$. 
From Theorem \ref{theo:generalRepTheo}, we know that 
$$
f_2(\cdot) = \sum_{i=1}^N \sum_{k_2=1}^{d_2} c_{i,k_2} K_2(\bx_i, \cdot) \boldsymbol{e}_{k_2}
$$ 
for certain coefficients $c_{i,k_2} \in \R$. 
Furthermore, we have that $f_1 \in \tilde{V}_1 = \operatorname{span}\{ K_1(f_2(\bx_i), \cdot) \mid i=1,\ldots,N\}$ and thus 
$$
f_1(\cdot) = \sum_{j=1}^N \alpha_j K_1\left(\sum_{i=1}^N \sum_{k_2=1}^{d_2} c_{i,k_2} K_2(\bx_i, \bx_j) \boldsymbol{e}_{k_2}, \cdot\right).
$$
The concatenated function is then given by 
$
h(\cdot) := f_1 \circ f_2(\cdot) = \sum_{j=1}^N \alpha_j \mathcal{K}^2(\bx_j,\cdot)
$
with the composition kernel 
\begin{equation} \label{eq:concKernel}
\mathcal{K}^2(\bx,\by) = K_1\left(\sum_{i=1}^N \sum_{k_2=1}^{d_2} c_{i,k_2} K_2(\bx_i, \bx) \boldsymbol{e}_{k_2}, \sum_{i=1}^N \sum_{k_2=1}^{d_2} c_{i,k_2} K_2(\bx_i, \by) \boldsymbol{e}_{k_2}\right).
\end{equation}
Therefore, instead of considering the infinite-dimensional optimization problem of finding $f_1 \in {\cal H}_1$ and $f_2 \in {\cal H}_2$ that minimize 
$$
J(f_1,f_2) = \sum_{i=1}^N {\cal L}\left(y_i,f_1(f_2(\bx_i))\right) + \Theta_1\left(\| f_1 \|_{{\cal H}_1}^2\right) + \Theta_2\left(\| f_2 \|_{{\cal H}_2}^2\right),
$$ we can restrict ourselves 
to finding the $N + N\cdot d_2$ coefficients $\alpha_j, c_{i,k_2}$ for $i,j=1,\ldots,N$ and $k_2 = 1,\ldots,d_2$. 

Note at this point that the problem of finding these coefficients is highly nonlinear and becomes more complicated for a larger number of layers $L$.
While the corresponding problem of optimizing the outermost coefficients, i.e.~$\alpha_j$ for $j=1,\ldots,N$ in our example, is still convex if the loss ${\cal L}$ 
and the penalty terms $\Theta_1,\Theta_2$ are convex, the optimization of the inner coefficients, i.e.~$c_{i,k_2}$ for $i=1,\ldots,N$ and $k_2 = 1,\ldots,d_2$, is usually not convex anymore 
and can have many local minima. Here, finding a global minimum is an issue because standard (iterative) optimization methods strongly depend on the chosen initial value and 
usually just deliver some local minimum.

If the optimization functional $J$ is smooth, one can rely on a Newton-type minimizer such as BFGS to solve the underlying optimization problem. However, if one deals 
with nonsmooth loss functionals or penalty terms, one should resort to specifically designed stochastic gradient algorithms which fit the problem at hand, see e.g.~\cite{SmolaNIPS2016}.

It remains to note that our representer theorem covers much more than just interpolation or least-squares regression algorithms. In the same fashion as the 
standard representer theorem in \cite{Smola}, it can directly be applied to more 
involved settings such as regression with a concatenation of support vector machines for instance. To this end, just choose ${\mathcal L}$ to be the $\varepsilon$-insensitive loss function
and $\Theta_1(x) = \ldots = \Theta_L(x) = x$. Furthermore, the choice of the additive penalties $\Theta_1, \ldots, \Theta_L$ in \eqref{eq:generalMinFunct} is 
rather arbitrary and one could think of more complex interactions 
between the penalties for each function $f_l, l = 1,\ldots,L$, as long as the arguments in the proof of Theorem \ref{theo:generalRepTheo} remain valid.

\subsection{An infinite-sample representer theorem for concatenated kernel learning} \label{sec:InfRepTheo}
After deriving the representer theorem \ref{theo:generalRepTheo} for the case of multi-layer kernel approximations, we now extend our results to 
the case of infinitely many samples. This has to be understood in analogy to the results in chapter 5 of \cite{Steinwart}, where such an infinite-sample representer 
theorem is provided for the single-layer case. Although such a result can usually not directly be applied to a practical problem unless the distribution of the data points is known, 
it can serve as a cornerstone for the analysis of robustness with respect to a measure change and can lead to a-priori convergence results, see \cite{Steinwart}. 
We will restrict the loss function to be an $L$-times differentiable Nemitski loss for the following theorem.
For a definition, we refer to \cite{Steinwart} or our appendix, where we define 
an even more general type of {\it Nemitski vector loss}. Note that, when we refer to convexity or differentiability of Nemitski losses or reproducing kernels, this should always be 
understood with respect to the second argument, i.e.\ $\textrm{d} K(\bx,\bz)$ should be understood as $\frac{\partial}{\partial \bz} K(\bx,\bz)$. 
In the following, we denote by ${\cal B}(X,Y)$ the space of bounded linear operators from $X$ to $Y$, endowed with the standard operator norm.

\begin{theorem} \label{theo:RepTheoInfiniteSample}
 Let ${\cal H}_1, \ldots, {\cal H}_L$ and the domains and ranges of their elements be as in theorem \ref{theo:generalRepTheo} and 
 let $\lambda_1, \ldots, \lambda_L > 0$. Let, furthermore, the kernel $K_l$ of ${\cal H}_l$ fulfill $K_l \in C^{1}(D_l \times D_l)$ 
 together with
 \begin{align} \label{eq:Kernelprereq}
  \sup_{\bx \in D_l} \| K_l(\bx,\bx) \|_2 \leq c_l ~~~ \text{ and } ~~~ \sup_{\bx,\bz \in D_l} \| \mathrm{d} K_l(\bx,\bz) \|_{{\cal B}\left(D_l,\R^{d_l \times d_l}\right)} \leq c_l
 \end{align}
 for some $c_l < \infty$ and all $l=1,\ldots,L$.
  Let $\PP$ be a distribution on $\Omega \times R_1$ and let ${\cal L}: R_1 \times \R \to [0,\infty)$ be a convex, $\PP$-integrable and 
 $1$-times differentiable (w.r.t.\ the second variable) Nemitski loss such that the absolute value of the derivative is also a $\PP$-integrable 
 Nemitski loss, which fulfills
\begin{align*}
\left| {\cal L}^{(k)}(y,z) \right| \leq b_k(y)  +  h_k(|z|) & ~ \text{ for all } (y,z) \in R_1 \times \R
\end{align*}
for some $L_{1,\PP_{R_1}}$-integrable\footnote{Here, $\PP_{R_1}$ denotes the marginal distribution of $\PP$ w.r.t.\ the second variable.} $b_k: R_1 \to [0,\infty)$ and some increasing $h_k: [0,\infty) \to [0,\infty)$ for $k=0,1$.
 Then, if we assume that a set of minimizers $\left(f_l\right)_{l=1}^L$ with $f_l \in {\cal H}_l$ of 
 \begin{equation} \label{eq:MinFunctInfiniteSample}
 J(f_1,\ldots,f_L) := \int_{\Omega \times R_1} {\cal L}\left(y,f_1 \circ \ldots \circ f_L(\bx)\right) \ \mathrm{d} \PP(\bx,y) + 
 \sum_{l=1}^L \lambda_l \| f_l \|^2_{{\cal H}_l}
 \end{equation}
exists, it fulfills the Bochner-type integral equation
\begin{equation} \label{eq:RepresenterIntegral}
 f_{l}(\cdot) = - \frac{1}{2 \lambda_i} \int_{\Omega \times R_1} K_l\left(\cdot, f_{l+1} \circ \ldots \circ f_L(\bx)\right) A_{f_l,f_{l+1},\ldots,f_{L}}(\bx,y) \ \mathrm{d} \PP(\bx,y)
\end{equation}
for some $A_{f_l,f_{l+1},\ldots,f_{L}} \in L_{1,\PP}(\Omega \times R_1;R_l)$ for all $l =1,\ldots,L$.
\end{theorem}
\begin{proof}
 The proof works layer-wise and it is an extension of the proof of theorem 5.8 of \cite{Steinwart}
 to the multi-layer case and to Nemitski vector loss functions, see also definition \ref{def:NemVec}.
 Let $g_i \in {\cal H}_i$ be arbitrary for all $i=1,\ldots,L$.
 Let $G_1 : \Omega \times R_1 \to R_2 \times R_1$ be defined by $G_1(\bx,y) = (g_2 \circ \ldots \circ g_L(\bx), y)$. Obviously, $G_1$ is a measurable map 
 and we can define the pushforward $G_{1,\star}(\PP)$ of $\PP$ onto $R_2 \times R_1$. With this we obtain 
 $$
 \int_{\Omega \times R_1} {\cal L}\left(y,g_1 \circ \ldots \circ g_L(\bx)\right) \ \mathrm{d} \PP(\bx,y) = \int_{R_2 \times R_1} {\cal L} \left(y, g_1(\bxi)\right) 
 \ \mathrm{d} G_{1,\star}(\PP)(\bxi,y).
 $$
 Now, with the functional $J_{g_2,\ldots,g_L} : {\cal H}_1 \to [0,\infty)$ defined by
 $$
 J_{g_2,\ldots,g_L}(g_1) := \int_{R_2 \times R_1} {\cal L} \left(y, g_1(\bxi)\right) \ \mathrm{d} G_{1,\star}(\PP)(\bxi,y) + \lambda_1 \| g_1 \|_{{\cal H}_1}^2,
 $$
 we can reformulate the minimization problem as
 $$
 \min_{g_1 \in {\cal H}_1,\ldots,g_L \in {\cal H}_L} J(g_1,\ldots,g_L) = \min_{g_2 \in {\cal H}_2,\ldots,g_L \in {\cal H}_L} \left( \min_{g_1 \in {\cal H}_1} J_{g_2,\ldots,g_L}(g_1) \right) + \sum_{l=2}^L \lambda_l \| g_l \|_{{\cal H}_l}^2.
 $$
 Since $G_1$ leaves the second argument unchanged, it directly follows from the $\PP$-integrability that ${\cal L}$ is also a $G_{1,\star}(\PP)$-integrable Nemitski loss. 
 Therefore, the application of the infinite-sample representer theorem 5.8 in \cite{Steinwart} states that the minimizer $g_1^{\star}$ of $J_{g_2,\ldots,g_L}$ 
 can be written as
 \begin{align*}
 g_1^{\star}(\cdot) = & - \frac{1}{2 \lambda_1} \int_{R_2 \times R_1} {\cal L}^{(1)}\left(y, g_1^{\star}(\bxi)\right) K_1(\cdot, \bxi) \ \mathrm{d} G_{1,\star}(\PP)(\bxi,y) \\
 = & - \frac{1}{2 \lambda_1} \int_{\Omega \times R_1} {\cal L}^{(1)}\left(y, g_1^{\star} \circ g_2 \circ \ldots \circ g_L(\bx) \right) K_1(\cdot, g_2 \circ \ldots \circ g_L(\bx)) \ \mathrm{d} \PP(\bx,y),
 \end{align*}
 where ${\cal L}^{(1)}$ denotes the first derivative of ${\cal L}$ w.r.t.\ the second argument. For the choice $g_i = f_i$ for $i=2,\ldots,L$,
 we obtain the minimizer $f_1 = g_1^{\star}$. Note that $f_1$ is continuous and $\| f_1 \|_{\infty} := \sup_{\bx \in D_1} |f_1(\bx)| < \infty$
 since ${\cal H}_1 \hookrightarrow C(D_1)$ follows directly by \eqref{eq:Kernelprereq}.
 Therefore, \eqref{eq:RepresenterIntegral} is true for $l=1$ since 
 \begin{align*}
 \left|A_{f_1,\ldots,f_{L}}(\cdot)\right| := & \ \left|{\cal L}^{(1)}\left(y, f_1 \circ f_2 \circ \ldots \circ f_L(\cdot) \right) \right|
 \leq b_1(y)  +  h_1\left(\left|f_1 \circ f_2 \circ \ldots \circ f_L(\cdot)\right|\right)  \\
 \leq & \ b_1(y) + h_1\left(\| f_1 \|_{\infty}\right)
 \end{align*}
 is in $L_{1,\PP}$ since $b_1 \in L_{1,\PP_{R_1}}(R_1)$.
 
 To tackle the next layer, we define $\tilde{\cal L} : R_1 \times R_2 \to [0,\infty)$ by 
 $$
 \tilde{\cal L}(y,\bz) := {\cal L}(y,f_1(\bz)).
 $$
 We proceed by showing that $\tilde{\cal L}$ is a $\PP$-integrable and $1$-times differentiable Nemitski vector loss. Then we show that we can use analogous techniques as 
 in \cite{Steinwart} - but for vector-valued functions - to ensure the representation \eqref{eq:RepresenterIntegral} for $l=2$.
 These arguments can then be iterated until we reach the innermost layer 
 and the proof is completed. Since the details are quite technical, we outsourced them into appendix \ref{appendix}.
\end{proof}

Theorem \ref{theo:RepTheoInfiniteSample} states that the solution $f_l$ in the $l$-th layer of \eqref{eq:MinFunctInfiniteSample} is an element of the range of the integral operator 
defined by the kernel $K_l\left( \cdot, f_{l+1} \circ \ldots \circ f_L(\cdot) \right): D_l \times \Omega \to \R^{d_{l} \times d_{l}}$. 
Note that the statement of theorem \ref{theo:generalRepTheo} can be derived by choosing a sum of finitely many Dirac measures $\delta_{\bx_i,y_i}$ as $\PP$ in theorem \ref{theo:RepTheoInfiniteSample}.
In this special case, the result boils down to $f_l$ being in the span of the kernel evaluations in the data points.

Note furthermore that - in contrast to the finite sample case - $f_l$ is defined as a convolution with the asymmetric kernel in \eqref{eq:RepresenterIntegral}.
This can be interpreted as a smoothing step for many kernel choices.
In this sense, we can expect the solutions $f_l$ of \eqref{eq:MinFunctInfiniteSample} to employ a higher degree of smoothness 
than in the case of \eqref{eq:generalMinFunct}, where the solutions are only finite linear combinations of kernels. However, this of course comes at the cost of 
the regularity condition on the kernels in the requirements of theorem \ref{theo:RepTheoInfiniteSample}.

\subsection{Relation to neural networks and deep learning}
\label{subsec:NNconnection}

We now come back to the finite sample case in this section and
discuss the relation of our representer theorem \ref{theo:generalRepTheo} to two of the most common approaches in deep learning with kernels, namely
multi-layer multiple kernel learning (MLMKL) and deep kernel networks (DKN),
see e.g.~\cite{ChoSaul1,LawrenceDeepGauss,Rebai,Strobl,Wilson,Zhuang}.
For reasons of simplicity, we restrict ourselves to the two-layer case $L=2$ here. 

\subsubsection{Relation to hidden layer neural networks}

Let us first illustrate how our approach can be encoded as a hidden layer feed-forward neural network. 
The idea behind artificial neural networks is the same as for multi-layer kernel learning, namely using concatenations of functions to compute good approximations. 
More precisely, the so-called universal approximation theorem states that already a two-layer neural network can approximate any continuous function arbitrarily well,
see \cite{Cybenko,Hornik}.
For more details on artificial neural networks and deep learning, we refer the reader
to \cite{BengioDeep}. 

As mentioned in the two-layer case above, we are aiming to find a function $h(\cdot) = f_1 \circ f_2(\cdot) = 
\sum_{j=1}^N \alpha_j \mathcal{K}^2(\bx_j,\cdot)$ with $f_1 \in {\cal H}_1$ and $f_2 \in {\cal H}_2$ and associated $K_1$ and $K_2$, respectively, 
where the kernel ${\mathcal{K}^2}$ is given by \eqref{eq:concKernel}.
The construction of $h$ can be easily encoded as a feed-forward neural network with one hidden layer if $K_1$ is a radial basis function (RBF) kernel for instance\footnote{For many other types of kernels, e.g.~tensor products of RBF kernels, one can still 
construct a more complex Sigma-Pi neural network for the computation of the output values.}.
We illustrate\footnote{Note that we only choose $d_2=1$ for illustrative reasons. For $d_2 > 1$, a neural network can be
built analogously with an additional hidden layer to compute the norm 
of the difference of $d_2$-dimensional vectors. However, this additional layer, which just computes $\| \bx - \by \|_2$ for given $\bx$ and $\by$,
has fixed weights and does not play any role for the optimization of the neural network.}
the case $d_2=1$ with an RBF kernel $K_1(z_1,z_2) = a(|z_1 - z_2|)$ for some function $a : \R \to \R$ in Figure \ref{fig:NN}. The first layer is 
split into the input layer with values $K_2(\bx_i,\bx)$ for $i=1,\ldots,N$ and an artificial ``always on'' layer with neuron-clusters that 
supply the constant values $K_2(\bx_i,\bx_j)$ with weights $-c_j$ for $i,j=1,\ldots,N$. Note that the $i$-th cluster $K_2(\bx_i, \bx_j)$ 
of the ``always on'' layer is only connected to the $i$-th neuron of the hidden layer. Note furthermore that the inputs $K_2(\bx_i,\bx)$ can also easily be computed 
by a neural network with fixed weights if $K_2$ is a radial basis kernel.
If we consider a ``deeper'' concatenation, we would need a deeper neural network with additional layers, i.e.~for $f_1 \circ \ldots \circ f_L$, we 
need $L-1$ hidden layers.

\begin{figure}[tb]
\def\layersep{3.5cm}
 \centering
\includegraphics{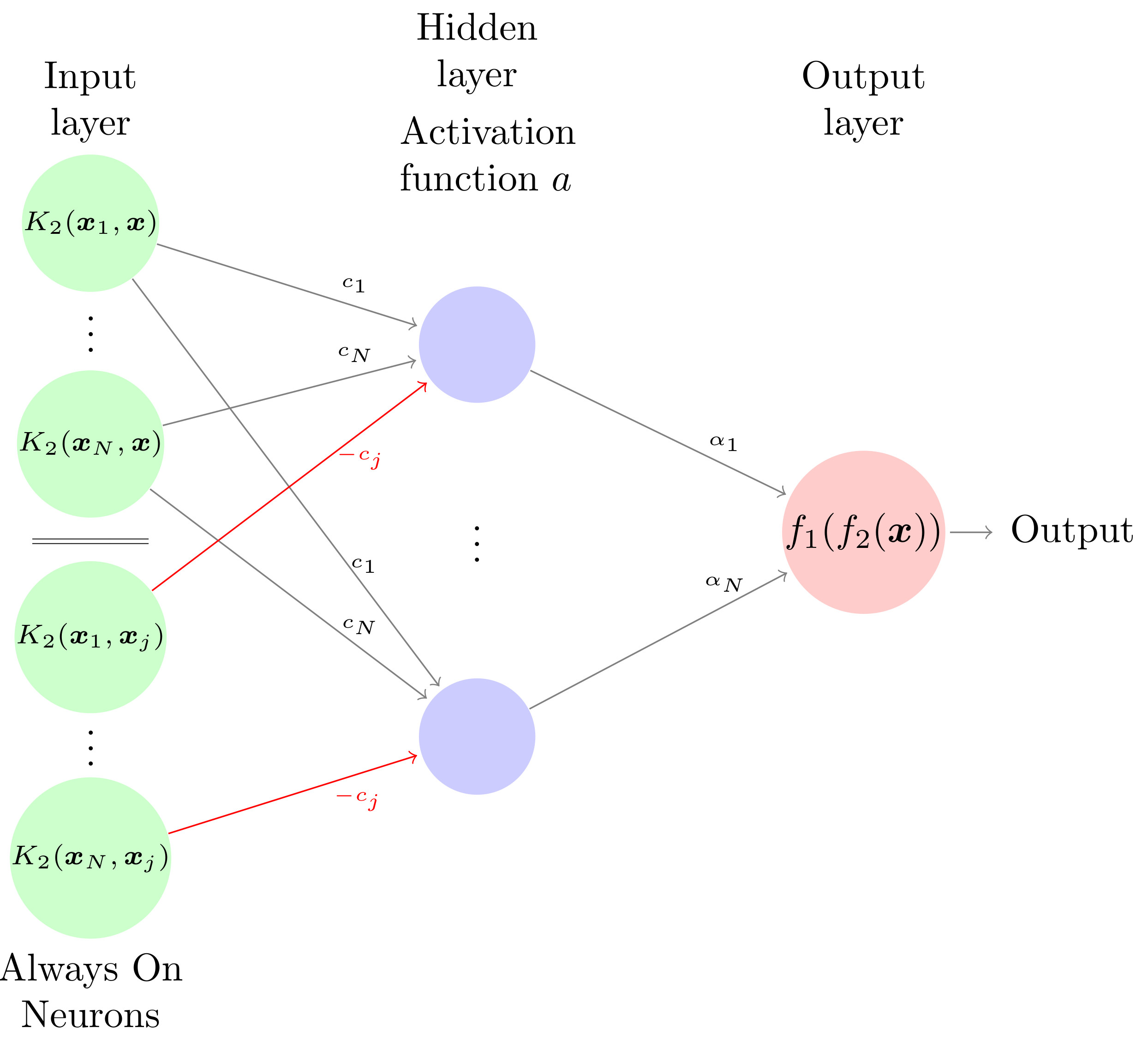}
\caption{A hidden layer, feed-forward neural network to simulate the concatenation of two functions $f_1$ and $f_2$ from reproducing kernel Hilbert spaces. 
For reasons of readability, we choose $d_2=1$ and write $c_i := c_{i,1}$.
The outer kernel is $K_1(z_1,z_2) = a(| z_1 - z_2|)$. Note that the $i$-th artificial ``always on'' neuron-cluster in the lower half of 
the first layer is 
written as $K_2(\bx_i,\bx_j)$, which stands for $N$ single neurons with values $K_2(\bx_i,\bx_1), \ldots, 
K_2(\bx_i,\bx_N)$. The cluster $K_2(\bx_i,\bx_j)$ is only connected to the $i$-th neuron of the hidden layer with weights $-c_j$ (red lines). This means that 
the value $\sum_{j=1}^N -c_j K_2(\bx_i,\bx_j)$ is forwarded to the $i$-th neuron of the hidden layer.}
\label{fig:NN}
\end{figure}

\subsubsection{Relation to multi-layer multiple kernel learning}
The common idea in MLMKL methods is to learn a kernel $\tilde{K}$, which consists of a chain of linear combinations of 
functions and an inner kernel, e.g. 
\begin{equation*}
\tilde{K}(\bx,\by) = \sum_{\ell = 1}^{n_1} \nu_{1,\ell} k_{1,\ell}\left(\sum_{i=1}^{n_2} \nu_{2,i} K_{2,i}( \bx, \by) \right)
\end{equation*}
in the two-layer case, where $k_{1,\ell}$ 
are real-valued functions for $\ell=1,\ldots,n_1$ and $K_{2,i}$ are different scalar-valued kernels for $i=1,\ldots,n_2$.
Note that the functions $k_{1,\ell}$ are chosen such that $\tilde{K}$ is still a kernel. In the case of linear $k_{1,\ell}$, \cite{DinuzzoPhd} has shown that 
the resulting algorithm becomes a standard MKL procedure and can be interpreted as a two-layer kernel network with a linear outer kernel.
However, for arbitrary $k_{1,\ell}$ this is not the case and we are dealing with a true MLMKL approach.
The specific MLMKL algorithm then aims to find the optimal values for the coefficients $\nu_{1,\ell}, \nu_{2,i}$ in order to determine the 
best $\tilde{K}$ for a regression of the given data $X$ and $Y$ with e.g.~a support vector regression algorithm.
Note that the kernels and the $k$-functions are usually chosen heuristically, e.g.~as polynomials,
Gaussians, sigmoidals, etc., see \cite{Rebai,Zhuang}. 

To apply our result to the two-layer MKL method above, let us consider the case $n_1 = 1$ and $n_2 = N$. We set $\nu_{1,1} = 1$ without loss of generality.
We let the outer function $k_{1,1}(z) = a(| z |)$ be the radial basis function used for the outer kernel (i.e.~middle layer) in Figure \ref{fig:NN}. Furthermore, we set 
$$
K_{2,i}(\bx,\by) := K_2(\bx_i, \bx) - K_2(\bx_i,\by).
$$
Note that the $K_{2,i}$ are no longer kernels anymore in this setting. However, they are now directly connected to our concatenated function learning approach since 
\begin{align*}
 \tilde{K}(\bx,\by) & = k_{1,1}\left( \sum_{i=1}^N \nu_{2,i} K_{2,i}(\bx,\by) \right) = a\left(\left|\sum_{i=1}^N \nu_{2,i} K_2(\bx_i, \bx) - \nu_{2,i} 
 K_2(\bx_i,\by) \right| \right) \\
 & = K_1\left(\sum_{i=1}^N \nu_{2,i} K_2(\bx_i,\bx), \sum_{i=1}^N \nu_{2,i} K_2(\bx_i,\by) \right) = \mathcal{K}^2\left(\bx,\by\right)
\end{align*}
from \eqref{eq:concKernel} with $c_i = \nu_{2,i}$ and the kernels $K_1$ and $K_2$ used in Figure \ref{fig:NN}. 
Altogether, we thus see that an MLMKL algorithm with these parameters already determines the optimal solution (provided that the right hand side of \eqref{eq:generalRepTheo} is solved exactly)
among all functions of type $h = f_1 \circ f_2$ with $f_1 \in {\cal H}_1$ and $f_2 \in {\cal H}_2$ according to Theorem \ref{theo:generalRepTheo}. 
This way, our representer theorem for concatenated functions directly applies to a special case of MLMKL networks.
Note however that a generalization of our arguments to more layers, i.e.~$L > 2$, is not straightforward for MLMKL.

\subsubsection{Relation to deep kernel learning approaches}
The class of DKN methods consists of algorithms which build a kernel by nonlinearly transforming the input vectors before applying an outer kernel function. 
This is in contrast to the MLMKL approach, where only the innermost function is a two-variate kernel and its evaluations are modified by some nonlinear outer functions. 
The models in this class range from simple feature map powers for some function $\Psi$, 
i.e.
$$\tilde{K}(x,y) = \underbrace{\Psi \circ \ldots \circ \Psi}_{L-1 \text{ times}}(\bx) \cdot \underbrace{\Psi \circ \ldots \circ \Psi}_{L-1 \text{ times}}(\by),$$
see \cite{ChoSaul1}, to more general variants like 
$$
\tilde{K}(x,y) = K\left(f_2 \circ \ldots \circ f_L(\bx),f_2 \circ \ldots \circ f_L(\by)\right)
$$
with nonlinear functions $f_2,\ldots,f_L$, see \cite{Wilson}. If we assume that $f_l \in {\cal H}_l$ for $l=2,\ldots,L$ stem from reproducing kernel Hilbert spaces 
with associated kernels $K_l$, 
we can apply Theorem \ref{theo:generalRepTheo} to this approach and obtain that each $f_l$ can be written as a finite linear combination of evaluations 
of the kernel $K_l$. Thus, we can directly apply our representer theorem for $L$-layer DKN algorithms.

\subsection{The two-layer interpolation problem}
\label{subsec:ConcInterpolation}

After analyzing the general multi-layer kernel concatenation problem in Theorem \ref{theo:generalRepTheo}, we now have a closer, more detailed 
look at the main component of it, namely the {\em concatenation of two functions}.
To this end, we specifically consider the interpolation problem for $L=2$. 
This simple, illustrative setting gives further insights into the way concatenation works in machine learning problems.

\subsubsection{Definition of the problem}

We slightly adapt our notation to this special case to obtain a direct relation to the single-layer interpolation problem from 
Section \ref{sec:IntRKHS}.
To this end, let $D:=d_2$ and consider the domain 
$\Phi := D_1 \subseteq \R^{D}$ together with the two function spaces
\begin{align*}
\outerRKHS := {\cal H}_1 &\subset C(\Phi):=\left\{f:\Phi \to \R \ | \ f \text{  continuous} \right\} &  \text{-- ``outer'' space,}\\
\innerRKHS := {\cal H}_2 &\subset \left\{ \boldsymbol{g}=\left(g_{1},\dots,g_{D}\right)^{T} :\Omega \to \Phi \ | \ 
\boldsymbol{g} \text{  continuous} \right\} & ~~~~~~~~ \text{-- ``inner'' space}.
\end{align*} 
Both spaces are supposed to be reproducing kernel Hilbert spaces, i.e.~there is an (outer) kernel $K := K_1:\Phi \times \Phi \to \R$ for $\outerRKHS$ such that
\begin{align*}
K\left(\bx,\cdot \right) &\in \outerRKHS && \text{for all } \bx \in \Omega, \\
f\left( \bx \right)&=\left(f,K\left(\bx,\cdot\right)  \right)_{ \outerRKHS }  && \text{for all } \bx \in \Omega \text{  and all  } f \in   \outerRKHS.
\end{align*}
The function space $\innerRKHS$ is assumed to be a vector-valued RKHS, i.e.~there is an (inner) kernel $\boldsymbol{K}:\Omega \times \Omega \to \R^{D \times D}$ such that
\begin{align*}
\boldsymbol{K}\left(\bx,\cdot \right) \bc &\in  \innerRKHS && \text{for all  } \bx \in \Omega \text{ and all } \bc \in \R^{D}, \\
\bc^T \boldsymbol{g}\left( \bx \right)&= \left(\boldsymbol{g},\boldsymbol{K}\left(\bx,\cdot \right) \bc \right)_{ \innerRKHS   }
&& \text{for all  } \bx \in \Omega \text{, all } \bc \in \R^{D} \text{ and all } \boldsymbol{g} \in  \innerRKHS.
\end{align*}

To formulate the concatenated interpolation problem in the spirit of \eqref{eq:interpolStd}, we have to define an appropriate functional and propose an appropriate 
search set for the minimization task.
To this end, we consider the functional $J: \outerRKHS \times  \innerRKHS \to \R$ given by
\begin{equation*}
J\left( f ,\boldsymbol{g}\right):=\left\|f\right\|^{2}_{ \outerRKHS } + \left\|\boldsymbol{g} \right\|^{2}_{\innerRKHS },
\end{equation*}
which penalizes the norms of both the outer and the inner function,
and the admissible set
\begin{equation*}
\mathcal{A}_{X,Y}:=\left\{\left(f,\boldsymbol{g} \right) \in  \outerRKHS \times  \innerRKHS  
\ | \ f \circ \boldsymbol{g} \left(\bx_{j} \right)=y_{j} \  1 \le j \le N \right\} \subset  \outerRKHS \times  \innerRKHS,
\end{equation*}
i.e.~the set of all concatenations of functions from $\outerRKHS$ and $\innerRKHS$ which interpolate the data.
With this notation, we can define the following variational optimization problem
\begin{equation*} \label{eq:interpolConc} \tag{P}
J\left(f,\boldsymbol{g}\right) \rightarrow \min  ~ \text{ for } \left(f,\boldsymbol{g}\right) \in  \mathcal{A}_{X,Y}
\end{equation*}

As explained in Section \ref{sec:IntRKHS}, the solution $f^*_{X,Y}$ to the standard interpolation problem \eqref{eq:interpolStd} can be computed by solving 
the system \eqref{eq:LGSRepresenterTheoremInt} of linear equations for a given set of fixed and pairwise disjoint input data points $X:=
\left\{\bx_{1},\dots,\bx_{N} \right\}$. Therefore, if we assume for a moment the inner function $\boldsymbol{g}$ in \eqref{eq:interpolConc} to be fixed and 
$Z := \boldsymbol{g}(X) = \left\{\bz_i = \boldsymbol{g}(\bx_i) \mid i = 1,\ldots,N \right\}$, then we obtain that the solution $f^*_{Z,Y}$ to \eqref{eq:interpolStd}
with data points $Z$ 
is the only admissible minimizer of the concatenated interpolation problem \eqref{eq:interpolConc}, i.e. 
\begin{equation*}
 f^*_{Z,Y} = \argmin{f \in \{ h \in \outerRKHS \mid (h, \boldsymbol{g}) \in {\mathcal{A}_{X,Y}}\}} \| f \|_{\outerRKHS}^2.
\end{equation*}
Note that the coefficients $\boldsymbol{\alpha}^* \in \R^N$ of $f^*_{Z,Y} = \sum_{i=1}^{N} \alpha^*_i K(\bz_i, \cdot)$ can be computed by solving the system 
$$
\boldsymbol{M}_{Z,Z} \boldsymbol{\alpha}^* = \boldsymbol{y} 
$$
and the value of the optimal energy, i.e.~the squared norm, is given by
\begin{equation*}
 \left\| f^*_{Z,Y} \right\|^{2}_{ \outerRKHS } = {\boldsymbol{\alpha}^*}^{T} \boldsymbol{M}_{Z,Z} \boldsymbol{\alpha}^* = 
 \boldsymbol{y}^{T}   \boldsymbol{M}^{-1}_{Z,Z} \boldsymbol{y}.  
\end{equation*}

\subsubsection{Application of the representer theorem}

In order to rewrite the concatenated interpolation problem \eqref{eq:interpolConc} into an unconstrained minimization problem by applying the above result, 
 we first have to 
 discuss what happens if $\boldsymbol{g}\left(\bx_{j}\right)=\boldsymbol{g}\left(\bx_{k}\right)$ for two indices $j \neq k$.
If equality holds also for the corresponding values from $Y$, i.e.~$y_{j}=y_{k}$, we can simply remove the pair $(x_j,y_j) \in X \times Y$ from the input data
and with it also the corresponding condition from the admissible set.
However, if $y_{j}\neq y_{k}$, there cannot be an $f \in \outerRKHS$ such that $\left(f,\boldsymbol{g} \right)\in \mathcal{A}_{X,Y}$.
In this case, we simply set $J\left(f,\boldsymbol{g}\right)=\infty$. 
Using this convention,
we can recast \eqref{eq:interpolConc} into the unrestricted optimization problem
\begin{equation*} \label{eq:interpolConcUnrestricted} \tag{uP}
J\left(f^*_{\boldsymbol{g}\left(X \right),Y},\boldsymbol{g} \right)= \boldsymbol{y}^{T} \boldsymbol{M}^{-1}_{\boldsymbol{g}\left(X \right),
\boldsymbol{g}\left(X \right)} \boldsymbol{y} + \left\|\boldsymbol{g} \right\|^{2}_{\innerRKHS } \rightarrow \min  \quad \text{for } \boldsymbol{g} \in \innerRKHS.
\end{equation*}
Therefore, we only have to consider the minimization with respect to $\boldsymbol{g} \in \innerRKHS$ since the optimal outer function $f^*_{\boldsymbol{g}(X),Y}$ is completely 
determined by the inner function values $\boldsymbol{g}(X)$ and $Y$.

Note that the side condition $\boldsymbol{g}\left(\bx_{j}\right) \neq \boldsymbol{g}\left(\bx_{k}\right)$ for $j \neq k$ can also be enforced by adding a penalty term of type
$
\sum_{i < j} W\left(\| \boldsymbol{g}(\bx_i) - \boldsymbol{g}(\bx_j) \|^2_2\right)
$
to $J$, where $W$ is a smooth function with $W(0) = \infty$, e.g.~$W(x) = \coth(x)$.
This can also remedy the problem of small condition numbers of $\boldsymbol{M}_{\boldsymbol{g}\left(X \right),
\boldsymbol{g}\left(X \right)}$ for large sample sizes since it maximizes distances between the point evaluations of $\boldsymbol{g}$.
Adding this to \eqref{eq:interpolConcUnrestricted}, we obtain
\begin{align} \label{eq:interpolConcUnrestrictedWithPenalty}
 J_{\gamma}\left(f^*_{\boldsymbol{g}\left(X \right),Y},\boldsymbol{g} \right) :=  J\left(f^*_{\boldsymbol{g}\left(X \right),Y},\boldsymbol{g} \right) + & \, \gamma 
 \sum_{1 \leq i < j \leq N} \coth\left(\left\| \boldsymbol{g}(\bx_i) - \boldsymbol{g}(\bx_j) \right\|^2_2 \right) \\ 
 & \quad \quad \rightarrow \min \quad \text{for } \boldsymbol{g} \in \innerRKHS. \nonumber
\end{align}
However, since using $J_{0} = J$ in our experiments in Section \ref{sec:Experiments} works out already well and the side condition does not seem to affect 
the results for moderate sample sizes, we restrict ourselves to the problem \eqref{eq:interpolConcUnrestricted} in the following.

Although the above considerations seem to simplify the concatenated interpolation problem, we still have to solve a highly nonlinear optimization problem 
over the (possibly) infinite-dimensional RKHS $\innerRKHS$. 
Nonetheless, by applying Theorem \ref{theo:generalRepTheo} to the unrestricted concatenated interpolation problem \eqref{eq:interpolConcUnrestricted}, we can restrict 
the search space $\innerRKHS$ to the span of the kernel translates in the input data.

\begin{corollary} \label{corr:representerInt}
 Let $V_X := \operatorname{span}\{\boldsymbol{K}(\bx_i,\cdot) \boldsymbol{e}_j \mid i=1,\ldots,N \text{ and } j=1,\ldots,D\}$, where $\boldsymbol{e}_j$ 
 denotes the $j$-th unit vector in $\R^D$. 
 Then, the solution $\boldsymbol{g}^*$ to the unconstrained concatenated interpolation problem \eqref{eq:interpolConcUnrestricted} fulfills $\boldsymbol{g}^* \in V_X$.
\end{corollary}
\begin{proof}
 We apply Theorem \ref{theo:generalRepTheo} with $L = 2$, $\Theta_1(x) = \Theta_2(x) = x$ and
$$
{\cal L}\left(y_i,f \circ \boldsymbol{g}(\bx_i)\right) = \left\{ \begin{array}{cl}
                                                                                                                                             0 & \text{ if } f \circ \boldsymbol{g}(\bx_i) = y_i  \\
                                                                                                                                             \infty & \text{ else,}
                                                                                                                                            \end{array}
\right.$$
which exactly resembles the interpolation problem \eqref{eq:interpolConcUnrestricted}.
\end{proof}

Due to Corollary \ref{corr:representerInt}, we can recast the unrestricted concatenated interpolation problem \eqref{eq:interpolConcUnrestricted} into 
\begin{equation*} \label{eq:interpolConcFiniteDim} \tag{uP-X}
J\left(f^*_{\boldsymbol{g}\left(X \right),Y},\boldsymbol{g} \right) = \boldsymbol{y}^{T} \boldsymbol{M}^{-1}_{\boldsymbol{g}\left(X \right), 
\boldsymbol{g}\left(X \right)} \boldsymbol{y} + \left\|\boldsymbol{g} \right\|^{2}_{\innerRKHS } \rightarrow \min  \text{ for } \boldsymbol{g} \in V_{X} \subset \innerRKHS.
\end{equation*}
This is a nonlinear, finite-dimensional and unrestricted optimization problem. 
We fix the kernel basis 
$\left\{ \boldsymbol{K}\left(\bx_{j}, \cdot \right)\boldsymbol{e}_{\ell}  \mid  (j,\ell) \in \mathcal{I} \right\}$ with 
$\mathcal{I} := \left\{\left(j,\ell \right)\ \in \N^2 \mid 1 \le j \le N, 1 \le \ell \le D \right\}$ 
to solve \eqref{eq:interpolConcFiniteDim}.
Then, the optimal solution can be written as 
\begin{equation} \label{eq:optimalgInt}
\boldsymbol{g}^*(\cdot)=\sum_{(j,\ell) \in \mathcal{I}} c^*_{j,\ell}\boldsymbol{K}\left(\bx_{j},\cdot \right)\boldsymbol{e}_{\ell}.
\end{equation}
In order to express the minimization problem \eqref{eq:interpolConcFiniteDim} with respect to the coefficients $\bc^* = \left(c^*_{1,1},\ldots,c^*_{N,D}\right)^T$,
we introduce
\begin{equation} \label{eq:QXX}
\boldsymbol{Q}_{X,X}\left(\bc \right) = \boldsymbol{M}_{\boldsymbol{g}(X),\boldsymbol{g}(X)} = 
\left(K \left( \sum\limits_{\left(j,\ell\right)\in I} c_{j,\ell} \boldsymbol{K}\left(\bx_{j},\bx_n\right)\boldsymbol{e}_{\ell}, \sum\limits_{\left(j,\ell\right)\in I}
c_{j,\ell} \boldsymbol{K}\left(\bx_{j},\bx_m\right)\boldsymbol{e}_{\ell}\right)\right)_{1\le n,m \le N}
\end{equation}
and the corresponding quadratic form 
\begin{equation*}
\mathcal{Q}:\R^{ND} \to \R, \quad \bc  \mapsto 
\boldsymbol{y}^{T}\boldsymbol{Q}_{X,X}\left(\bc \right)^{-1}
\boldsymbol{y}.
\end{equation*}
Furthermore, to express $\| \boldsymbol{g}^* \|_{\innerRKHS}^2$ with respect to $\bc^*$, we need
\begin{equation} \label{eq:NXX}
\mathcal{N}:\R^{ND} \to \R, \quad \bc \mapsto 
\sum_{j,k=1}^{N} \begin{pmatrix} c_{j,1}\\ \vdots \\ c_{j,D} \end{pmatrix}^{T}\boldsymbol{K}\left(\bx_{j},\bx_{k} \right) \begin{pmatrix} c_{k,1}\\ \vdots \\ c_{k,D} \end{pmatrix}.
\end{equation}
Finally, we obtain the finite-dimensional optimization problem
\begin{equation*}
\label{eq:interpolConcFiniteDimVec} \tag{Int}
\bc^* = \argmin{\bc \in \R^{ND}} \underbrace{\mathcal{Q}\left(\bc \right)}_{\| f^*_{\boldsymbol{g}(X),Y} \|^2_{\outerRKHS}} + \underbrace{\mathcal{N}\left(\bc \right)}_{\| \boldsymbol{g} \|^2_{\innerRKHS}}.
\end{equation*}

\subsubsection{Solving the minimization problem}

The unconstrained problem \eqref{eq:interpolConcFiniteDimVec} is highly nonlinear because the coefficients $c_{j,\ell}$ are transformed by the outer kernel function $K$.
It can be tackled by any suitable iterative optimization algorithm. 
If the kernels $K$ and $\boldsymbol{K}$ are differentiable, a quasi-Newton approach is appropriate. 
If this is not the case, a derivative-free optimizer should be chosen.

Note that we can restrict the minimization in \eqref{eq:interpolConcFiniteDimVec} 
to a compact subset of $\R^{ND}$ without loss of generality. To this end, let $\underline{\boldsymbol{K}} \in \R^{ND \times ND}$ be 
the $N \times N$ matrix of matrices $\boldsymbol{K}(\bx_i,\bx_j) \in \R^{D \times D}$ and note that 
$$
\mathcal{Q}(\bc) + \mathcal{N}(\bc) \geq 0 + \lambda_{\min}\left(\underline{\boldsymbol{K}}\right) \cdot \| \bc \|_2^2 \stackrel{\| \bc \|_2 \to \infty}{\longrightarrow} \infty,
$$
where $\lambda_{\min}\left(\underline{\boldsymbol{K}}\right) > 0$ denotes the smallest eigenvalue of $\underline{\boldsymbol{K}}$. Therefore, we can restrict our search to the compact set 
$A := \left\{ \bc \in \R^{ND} \mid \| \bc \|_2 \leq C \right\}$ for a large enough $C > 0$. Unfortunately, we cannot directly obtain the existence of a minimizer from this since 
\eqref{eq:interpolConcFiniteDimVec} is not continuous.
 However, if we add a smooth term 
 \begin{align*}
 {\cal P}^{\gamma}(\bc) & = \gamma  \sum_{1 \leq m < n \leq N} \coth\left(\left\| \boldsymbol{g}(\bx_m) - \boldsymbol{g}(\bx_n) \right\|^2_2 \right) \\ 
 & = \gamma \sum_{1\leq m < n \leq N} 
 \coth\left( \left\| \sum_{(j,\ell) \in \mathcal{I}} c_{j,\ell} \left( \boldsymbol{K}\left(\bx_{j}, \bx_{m} \right) - \boldsymbol{K}\left(\bx_{j}, \bx_{n} \right) \right)
 \boldsymbol{e}_{\ell} \right\|_2^2 \right),
 \end{align*} which is equivalent to \eqref{eq:interpolConcUnrestrictedWithPenalty}, for $\gamma > 0$,
 we can deduce the existence of a minimizer with the direct method from the calculus of variations.
 To this end, note that for a minimizing sequence $\left(\bc_i\right)_{i=1}^{\infty}$ of $\mathcal{Q} + \mathcal{N} + \mathcal{P}^{\gamma}$, there necessarily exist $i_0 \in \N$ and $C_0 > 0$ 
 such that all mutual squared distances $\left\| \boldsymbol{g}(\bx_m) - \boldsymbol{g}(\bx_n) \right\|^2_2$ with $1\leq m < n \leq N$ are larger than $C_0$ for 
 all $\bc_i$ with $i > i_0$. Therefore, we can restrict the minimization to the compact subdomain 
 $$
 A \cap \left\{\bc \in \R^{ND} \mid \| \boldsymbol{g}(\bx_m) - \boldsymbol{g}(\bx_n) \|_2^2 \geq C_0 ~ \text{ for } 1 \leq m < n \leq N \right\},
 $$
 on which $\mathcal{Q} + \mathcal{N} + \mathcal{P}^{\gamma}$ is continuous, and the existence of a minimizer follows.
 Nevertheless, as we explained above, the critical condition ${\cal P}^{\gamma}(\bc) = \infty$ is practically never met 
 for moderate data set sizes and, therefore, it is safe to assume that there also 
exists a minimizer for \eqref{eq:interpolConcFiniteDimVec}.
Note however that, depending on the kernels and the data at hand, there usually might exist many minimizers and the solution to \eqref{eq:interpolConcFiniteDimVec} might not be unique.
To reduce the chance of getting stuck in a local minimum, we propose to restart the minimization procedure several times with different 
starting values for $\bc^*$.

Since we will be dealing with differentiable kernel functions in Section \ref{sec:Experiments} and since the derivatives of these kernels can be computed explicitly,
we propose a BFGS minimization algorithm to solve \eqref{eq:interpolConcFiniteDimVec}. To this end, note that the only derivatives we need are 
essentially the derivative of the inverse of $\boldsymbol{Q}_{X,X}(\bc)$, i.e.
$$
\frac{\partial}{\partial c_{m,n}} \boldsymbol{Q}_{X,X}^{-1}(\bc) = - \boldsymbol{Q}_{X,X}^{-1}(\bc) \frac{\partial}{\partial c_{m,n}} \boldsymbol{Q}_{X,X}(\bc) \boldsymbol{Q}_{X,X}^{-1}(\bc),
$$
and the derivative of $\boldsymbol{Q}_{X,X}(\bc)$. The latter consists of the derivative of the outer kernel $K$, which 
is known analytically for all kernel choices that we discuss in Section \ref{sec:Experiments}, 
and 
$$
\frac{\partial}{\partial c_{m,n}} \boldsymbol{g}(\bx) = \frac{\partial}{\partial c_{m,n}} \sum_{(j,\ell) \in \mathcal{I}} c_{j,\ell}\boldsymbol{K}\left(\bx_{j},\bx \right)\boldsymbol{e}_{\ell}
= \boldsymbol{K}\left(\bx_{m},\bx \right)\boldsymbol{e}_{n}
$$ 
for each $(m,n) \in \mathcal{I}$.
The overall computational cost complexity for one BFGS step, i.e.~the evaluation of ${\mathcal{Q}}, {\mathcal{N}}$ and their derivatives, is bounded by ${\cal O}\left(N^3D + (ND)^2\right)$.


\subsection{Two-layer Least-squares regression}
\label{subsec:ConcRegression}

After the discussion of the two-layer interpolation problem in the last section, we now consider the regularized two-layer least-squares problem in more detail. 
This is a natural extension of the two-layer least-squares problem RLS2 considered in \cite{DinuzzoPhd} to the case of nonlinear outer kernels.

\subsubsection{Definition of the problem}

For concatenated, regularized least-squares regression, the minimization task changes to
\begin{align*} \label{eq:regressionConc} \tag{R}
J_{\lambda,\mu}\left( f ,\boldsymbol{g}\right) := \sum_{j=1}^{N}\left|f \circ \boldsymbol{g}\left(\bx_{j} \right)-y_{j} \right|^{2} + &
\lambda \left\|f\right\|^{2}_{ \outerRKHS } + \mu \left\|\boldsymbol{g} \right\|^{2}_{\innerRKHS } \\
& \quad \rightarrow \min \text{ for } f \in \outerRKHS, g \in \innerRKHS \; \; \,
\end{align*}
with $\lambda, \mu > 0$, which is in the same fashion as the standard least-squares regression problem \eqref{eq:regressStd}.

Analogously to our considerations in Section \ref{subsec:ConcInterpolation}, we find that, for fixed inner points $Z = \boldsymbol{g}(X) \subset \Phi$, 
the function $f^{\lambda}_{Z,Y}$, see \eqref{eq:regressStd}, is the solution of the problem
\begin{equation*}
\sum_{j=1}^{N}\left|f(\bz_j) -y_{j} \right|^{2} + \lambda \left\|f\right\|^{2}_{ \outerRKHS } 
  \rightarrow \min  \text{ for } f \in \outerRKHS .
\end{equation*}
The corresponding coefficients $\boldsymbol{\alpha}^{\lambda} \in \R^N$ with respect to the basis $\{ K(\bz_j, \cdot) \mid j=1,\ldots,N \}$ 
are computed by solving 
$$
\left(\boldsymbol{M}_{Z,Z} + \lambda \boldsymbol{I} \right) \boldsymbol{\alpha}^{\lambda} = \boldsymbol{y}.
$$
Therefore, each of the terms of the optimal energy can be expressed as 
\begin{eqnarray*}
 \left\|f_{Z,Y}^{\lambda}\right\|^{2}_{ \outerRKHS }&=&{\boldsymbol{\alpha}^{\lambda}}^{T}  \boldsymbol{M}_{Z,Z} \boldsymbol{\alpha}^{\lambda} = 
\boldsymbol{y}^{T} \left( \boldsymbol{M}_{Z,Z} +\lambda \boldsymbol{I} \right)^{-1}  \boldsymbol{M}_{Z,Z} \left( \boldsymbol{M}_{Z,Z} +\lambda \boldsymbol{I} \right)^{-1} \boldsymbol{y}, \\
\sum_{j=1}^{N}\left|f_{Z,Y}^{\lambda}\left(\bz_{j} \right)-y_{j} \right|^{2}&=& 
\left\|\boldsymbol{M}_{Z,Z} \boldsymbol{\alpha}^{\lambda} - \boldsymbol{y}  \right\|_{2}^{2} =
\left\|\left(\boldsymbol{I}-\boldsymbol{M}_{Z,Z} \left( \boldsymbol{M}_{Z,Z} +\lambda \boldsymbol{I} \right)^{-1}\right) \boldsymbol{y}  \right\|_{2}^{2}.
\end{eqnarray*}

\subsubsection{Application of the representer theorem}

Analogously to \eqref{eq:interpolConcUnrestricted}, we can use 
\begin{align} \label{eq:regressionFuncinMatrixForm}
J_{\lambda,\mu}\left(f_{\boldsymbol{g}\left(X \right),Y}^{\lambda},\boldsymbol{g} \right) & = 
\lambda \boldsymbol{y}^{T} \left( \boldsymbol{M}_{\boldsymbol{g}\left(X \right), \boldsymbol{g}\left(X \right)} +\lambda \boldsymbol{I} \right)^{-1} 
\boldsymbol{M}_{ \boldsymbol{g}\left(X \right), \boldsymbol{g}\left(X \right) } \left( \boldsymbol{M}_{ \boldsymbol{g}\left(X \right), \boldsymbol{g}\left(X \right) }
+\lambda \boldsymbol{I} \right)^{-1}  \boldsymbol{y}  \nonumber \\
&+ \mu \left\|\boldsymbol{g} \right\|^{2}_{\innerRKHS } +\left\|\left(\boldsymbol{I}-\boldsymbol{M}_{ \boldsymbol{g}\left(X \right),\boldsymbol{g}\left(X \right) }
\left( \boldsymbol{M}_{ \boldsymbol{g}\left(X \right) , \boldsymbol{g}\left(X \right) } + \lambda \boldsymbol{I} \right)^{-1}\right) \boldsymbol{y}  \right\|_{2}^{2}
\end{align}
to reformulate \eqref{eq:regressionConc} as 
\begin{equation*} \label{eq:regressionConcUnrestricted} \tag{uR}
 J_{\lambda,\mu}\left(f_{\boldsymbol{g}\left(X \right),Y}^{\lambda},\boldsymbol{g} \right) \rightarrow \min  \text{ for } \boldsymbol{g} \in \innerRKHS.
\end{equation*} 
\begin{corollary} \label{corr:representerReg}
The solution $\boldsymbol{g}^{\lambda,\mu}$ to the unconstrained concatenated regression problem \eqref{eq:regressionConcUnrestricted} fulfills $\boldsymbol{g}^{\lambda,\mu} \in V_X$.
\end{corollary}
\begin{proof}
We apply Theorem \ref{theo:generalRepTheo} with $L = 2$, $\Theta_1(x) = \lambda \cdot x$, $\Theta_2(x) = \mu \cdot x$ and
$$
{\cal L}\left(y_i,f \circ \boldsymbol{g}(\bx_i)\right) = 
\left|f \circ \boldsymbol{g}\left(\bx_{i} \right)-y_{i} \right|^{2},
$$
which resembles the regression problem \eqref{eq:regressionConc}.
\end{proof}

Hence, as for interpolation, we obtain a representer theorem for concatenated least-squares regression, which allows us to replace
the infinite-dimensional optimization problem \eqref{eq:regressionConc} with the finite-dimensional problem
\begin{equation*} \label{eq:regressionConcFiniteDim} \tag{uR-X}
 J_{\lambda,\mu}\left(f_{\boldsymbol{g}\left(X \right),Y}^{\lambda},\boldsymbol{g} \right) \rightarrow \min  \text{ for } \boldsymbol{g} \in V_X \subset \innerRKHS.
\end{equation*}


Finally, we want to express \eqref{eq:regressionConcFiniteDim} in terms of the coefficients $\bc^{\lambda,\mu} = \left(c^{\lambda,\mu}_{1,1},\ldots,c^{\lambda,\mu}_{N,D}\right)^T$ 
of $\boldsymbol{g}^{\lambda,\mu}$ with respect to the basis $\left\{ \boldsymbol{K}\left(\bx_{j}, \cdot \right)\boldsymbol{e}_{\ell}  \mid  (j,\ell) \in \mathcal{I} \right\}$. 
To this end, we set $\boldsymbol{A} := \left(\boldsymbol{Q}_{X,X}\left(\bc \right) + \lambda \boldsymbol{I}\right)^{-1}$ and define the quadratic forms 
\begin{align*}
\mathcal{Q}^{\lambda}:\R^{ND} \to \R, & \quad \bc  \mapsto 
\lambda \cdot \boldsymbol{y}^{T} \boldsymbol{A} \boldsymbol{Q}_{X,X}\left(\bc\right) 
\boldsymbol{A} \boldsymbol{y}, \\
\mathcal{N}^{\mu}: \R^{ND} \to \R, & \quad \bc \mapsto \mu \cdot \mathcal{N}(\bc) \quad \text{ and} \\
\mathcal{C}^{\lambda}: \R^{ND} \to \R, & \quad \bc \mapsto \boldsymbol{y}^T \left(\boldsymbol{I}-\boldsymbol{Q}_{X,X}\left(\bc\right) \boldsymbol{A} \right)^T 
\left(\boldsymbol{I}-\boldsymbol{Q}_{X,X}\left(\bc\right) \boldsymbol{A} \right) \boldsymbol{y}
\end{align*}
with the help of \eqref{eq:QXX} and \eqref{eq:NXX}. 
Subsequently, we arrive at the optimization problem
\begin{equation*}
\label{eq:regressionConcFiniteDimVec} \tag{Reg}
\bc^{\lambda,\mu} = \argmin{\bc \in \R^{ND}} \mathcal{Q}^{\lambda}\left(\bc \right) + \mathcal{N}^{\mu}\left(\bc \right) + \mathcal{C}^{\lambda}\left( \bc \right),
\end{equation*}
which is the equivalent to \eqref{eq:regressionConcFiniteDim}.

\subsubsection{Solving the minimization problem}

Note that the existence of a minimizer follows by similar arguments as in the previous section for the interpolation problem, i.e.
$$
\mathcal{Q}^{\lambda}\left(\bc \right) + \mathcal{N}^{\mu}\left(\bc \right) + \mathcal{C}^{\lambda}\left( \bc \right) \geq 
\mu \cdot \lambda_{\min}\left(\underline{\boldsymbol{K}}\right) \cdot \| \bc \|_2^2 \stackrel{\| \bc \|_2 \to \infty}{\longrightarrow} \infty
$$
and we can thus restrict the search for a minimizer to a compact subset of $\R^{ND}$. 
For regression we need the inverse of $\boldsymbol{Q}_{X,X}(\bc) + \lambda \boldsymbol{I}$ to compute $\mathcal{Q}^{\lambda}$, which is positive definite 
for every $\lambda > 0$ and, therefore, there are no pathological cases as in the interpolation setting. Thus, the functions
${\mathcal{Q}^{\lambda}},\mathcal{N}^{\mu},\mathcal{C}^{\lambda}$ are continuous and the minimization of \eqref{eq:regressionConcFiniteDimVec} over a
compact subset of $\R^{ND}$ has a minimizer. Nevertheless, also in this case the minimizer is not necessarily unique.

While the optimization for the coefficients in the RLS2 algorithm proposed in \cite{DinuzzoPhd} boils down to a simplex-constrained linear least-squares problem, we have 
to deal with a high degree of nonlinearity here. Nevertheless, if the kernel functions are differentiable, we can again - as in the interpolation case - employ a BFGS algorithm with several restarts to approximately find the optimal coefficients $\bc^{\lambda,\mu}$. To this end, 
note that $\mathcal{Q}^{\lambda}$ and $\mathcal{N}^{\mu}$ can be computed similarly as $\mathcal{Q}$ and $\mathcal{N}$ in the interpolation case. Furthermore, also
the derivative of ${\mathcal{C}}^{\lambda}$ can be computed with the same techniques since we essentially only need the derivatives of $\boldsymbol{Q}_{X,X}(\bc)$ and
$\left( \boldsymbol{Q}_{X,X}\left(\bc\right) + \lambda \boldsymbol{I} \right)^{-1}$. While the number of terms is larger than in the interpolation case, 
the asymptotic computational runtime is still bounded by ${\cal O}\left(N^3D + (ND)^2\right)$. Furthermore, the condition number of the matrix $\boldsymbol{Q}_{X,X}\left(\bc\right) + \lambda \boldsymbol{I}$ 
is smaller than the one of $\boldsymbol{Q}_{X,X}$, which had to be inverted for interpolation. Therefore, computing $\mathcal{Q}^{\lambda}(\bc)$ 
with an iterative solver for the application of $\left(\boldsymbol{Q}_{X,X}\left(\bc\right) + \lambda \boldsymbol{I}\right)^{-1}$ needs fewer computational steps 
than computing ${\mathcal{Q}(\bc)}$ in the interpolation case.

Finally, let us remark that for both interpolation and least-squares regression there exists another possibility to obtain a finite-dimensional optimization problem from \eqref{eq:interpolConc} and \eqref{eq:regressionConc}, respectively, without using the representer theorem. 
We could discretize the functions $f_1 \in {\cal H}_1$ and $f_2 \in {\cal H}_2$ by $\tilde{f}_1 \in V_1$ and $\tilde{f}_2 \in V_2$ with finite-dimensional spaces 
$V_1,V_2$, see e.g.~\cite{BohnGriebel} for an error analysis of this scenario for single-layer regression. However, when following this approach, the choice of the specific discretization can severely influence the results of the minimization. 
Furthermore, we are limited by the size of the dimensions of the discretization spaces $V_1, V_2$, which influences the computational costs for solving the underlying optimization problem. 

\section{The effects of concatenated learning}
\label{sec:Experiments}

This section serves to illustrate the main operating principle behind the concatenated interpolation and regression algorithms presented in the previous section. Note that our brief 
considerations in this section are not meant to provide a thorough numerical analysis of the performance of the algorithms but are rather thought to aid the understanding of 
their internal mechanisms. For benchmarks of highly performant variants of our basic algorithms on real-world data we 
refer the interested reader to \cite{LawrenceDeepGauss,Rebai,Zhuang}.

\subsection{Kernel choice}
 For reasons of simplicity, we will stick 
to the two-layer case and to outer function spaces $\outerRKHS$ with associated kernel $K$ which are defined on the whole space $\R^D$. 
This way, the image ${\Phi}$ of the inner function space is automatically contained in the domain of the outer function space. 
Furthermore, if not stated otherwise, we assume that the matrix-valued kernel 
$\boldsymbol{K}:\Omega \times \Omega \to \R^{D \times D}$ of the inner RKHS can be written as 
\begin{equation} \label{eq:innerKernelChoice}
 \boldsymbol{K}(\bx,\boldsymbol{y}) =  K_{\mathcal{I}}(\bx,\boldsymbol{y}) \cdot \text{diag}(\boldsymbol{a})
\end{equation}
for some weight vector $\boldsymbol{a} \in \R_+^D$. Here, $\text{diag}(\boldsymbol{a})$ denotes the diagonal matrix $\boldsymbol A$ with $\boldsymbol{A}_{ii} = \boldsymbol{a}_i$ 
and $K_{\mathcal{I}} : \Omega \times \Omega \to \R$ is a scalar-valued kernel function.

Possible outer and inner kernel functions $K$ and $K_{\mathcal{I}}$ are the 
polynomial kernel 
$$
K_{\text{Poly},p}(\bx,\by) := \left( \bx^T \by + 1 \right)^{p},
$$
the Gaussian kernel
$$
K_{\text{Gauss},\sigma}(\bx,\by) := \exp\left(- \frac{\| \bx - \by \|^2}{2\sigma^2} \right)
$$
and the tensor-product Mat\'{e}rn kernel
$$
K_{\text{TensorMat\'{e}rn},s}(\bx,\by) := \prod_{i=1}^d \kappa_{\frac{2s-1}{2}}\left(| x_i - y_i |\right) \cdot | x_i - y_i |^{\frac{2s-1}{2}}.
$$
where $\kappa_\alpha$ denotes the modified (hyperbolic) Bessel function of the second kind with parameter $\alpha$. Note that the latter characterizes the Sobolev space 
of dominating mixed smoothness of order $s \in \N$, see e.g.~\cite{FasshauerYe,HarbrechtMatern} for a bi-variate version.
These Sobolev spaces play an important role for hyperbolic cross or sparse grid approximations for instance, see e.g.~\cite{BungartzGriebel}.
Note that the Gaussian kernel is already a tensor product kernel by nature. 

\subsection{Experiment design}

Let us choose $\Omega = [-1,1]^2$. We will evaluate our method for the two test functions
\begin{eqnarray*}
 h_1 : \Omega \to \R &~~& h_1(x,y) := (0.1 + |x - y|)^{-1} \\
 h_2 : \Omega \to \R &~~& h_2(x,y) := \left\{ \begin{array}{ll} 1 & \text{ if } x \cdot y > \frac{3}{20} \\ 0 & \text{ else } \end{array}
\right..
\end{eqnarray*}
The function $h_1$ employs a kink-like structure along the diagonal of the domain, while $h_2$ represents an indicator function with a jump. Neither of these two functions is an element of
a reproducing kernel space spanned by any of the above kernel functions
for arbitrary parameters $p,s \in \N, \sigma \in (0,\infty)$. Therefore they cannot be approximated too well by a single-layer method.
The approximation of such functions with kinks or jumps by (a composition of) smooth functions plays an important role in applications from econometrics, finance or two-phase flow problems
for example. 

We choose $D=d=2$, i.e.~$\Omega, {\Phi} \subset \R^2$, and $\boldsymbol{a} = (1~ 1)^T$. 
Then, we independently draw $N = 100$ random equidistributed points
$\left\{\bx_{1},\dots,\bx_{N}\right\} \subset \Omega$ and set $y_i := h_*(\bx_i) + \varepsilon_i$ for all $i = 1,\ldots,N$ for the function $h_* \in \{h_1,h_2\}$. Here, 
$\varepsilon_i$ are additive noise perturbations which are drawn i.i.d.~according to a centered Gaussian distribution with standard deviation $0.01$.
To solve \eqref{eq:interpolConcFiniteDimVec} or \eqref{eq:regressionConcFiniteDimVec}, respectively,
we use a BFGS algorithm with random initialization of the coefficient vector $\bc$ of the inner function, see also \eqref{eq:optimalgInt}.
As the goal functions employ many local minima, we run the algorithm 
sufficiently many times to achieve a good approximation to the global minimum. It turned out that $64$ runs were sufficient for our case of $100$ data points in $2$ dimensions.
From the $64$ runs we pick the vector $\bc$ (and with this the functions $f$ and $\boldsymbol{g}$) 
for which the smallest goal function value in \eqref{eq:interpolConcFiniteDimVec} 
or \eqref{eq:regressionConcFiniteDimVec}, respectively, is achieved.

To be able to compare our computed $f(g(\cdot))$, which approximates the true solution 
$f_{\boldsymbol{g}\left(X\right),Y}^*\left(\boldsymbol{g}^*\left(\cdot\right)\right)$ or $f_{\boldsymbol{g}\left(X\right),Y}^{\lambda}\left(\boldsymbol{g}^{\lambda,\mu}(\cdot)\right)$, 
respectively,
to the result of a standard kernel interpolation/regression, we also calculate the interpolant/regressor $w \in \left\{f^*_{X,Y},f^{\lambda}_{X,Y}\right\}$. This resembles 
the solution to \eqref{eq:interpolStd} or \eqref{eq:regressStd}, respectively, for the reproducing kernel Hilbert space $H(\Omega,\R)$ which employs the 
same kernel type and parameters as 
$\outerRKHS$ but on the domain $\Omega$ instead of $\Phi$.
We then define $\bt_i, i=1,\ldots,n_t,$ as the points of a uniform grid of meshwidth $\frac{1}{50}$ over $\Omega = [-1,1]^2$, i.e.~$n_t = 101^2$, 
and consider the \emph{pointwise error}
\begin{eqnarray*}
 | \left(f \circ \boldsymbol{g} - h_*\right)(\bt_i) | ~~~ \text{ and } ~~~  | \left(w - h_*\right)(\bt_i) |, 
\end{eqnarray*}
which we visualize in a two-dimensional contour plot. 

\subsubsection{Interpolation} \label{subsec:interpolatwolayerexample}
We first compare the results for two-layer interpolation, see \eqref{eq:interpolConcFiniteDimVec}, with the results for single-layer interpolation, see \eqref{eq:interpolStd}.
To this end, we choose an outer Mat\'{e}rn kernel $K = K_{\text{TensorMat\'{e}rn},s}$ with $s=1$ and 
an inner polynomial kernel $K_{\mathcal{I}} = K_{\text{Poly},p}$ with $p=1$ or $p=2$.
In Figure
\ref{fig:MatPolTPInt} we display the pointwise errors.
We observe that there is a visible improvement in the error when dealing with two-layer interpolation instead of single-layer 
interpolation. While the benefits of two-layer interpolation are already observable for the test function $h_2$, they become even more obvious for $h_1$. 
As explained in the beginning of Section \ref{sec:IntChained}, the fact that the kink of $h_1$ is not parallel to a coordinate axis poses a problem when dealing with the tensor-product 
kernel. Since a linear transformation (rotation) would suffice to remedy this problem, the polynomial kernel of degree $p=1$ already suffices to obtain a better error behavior. 
Therefore, $p=2$ can already lead to a small overfitting effect as we observe in Figure \ref{fig:MatPolTPInt}.
Nevertheless, the error is still significantly better than in the single-layer 
case. In the case of $h_2$, however, we have a jump along two nonlinear curves. Here, $p=2$ seems to be more appropriate to deal with this problem.
%
Overall, we come to the conclusion that interpolation in reproducing kernel Hilbert spaces can significantly benefit from a two-layer approach if the reproducing kernel 
at hand does not suit the underlying function. 


\begin{figure}
\centering
\includegraphics{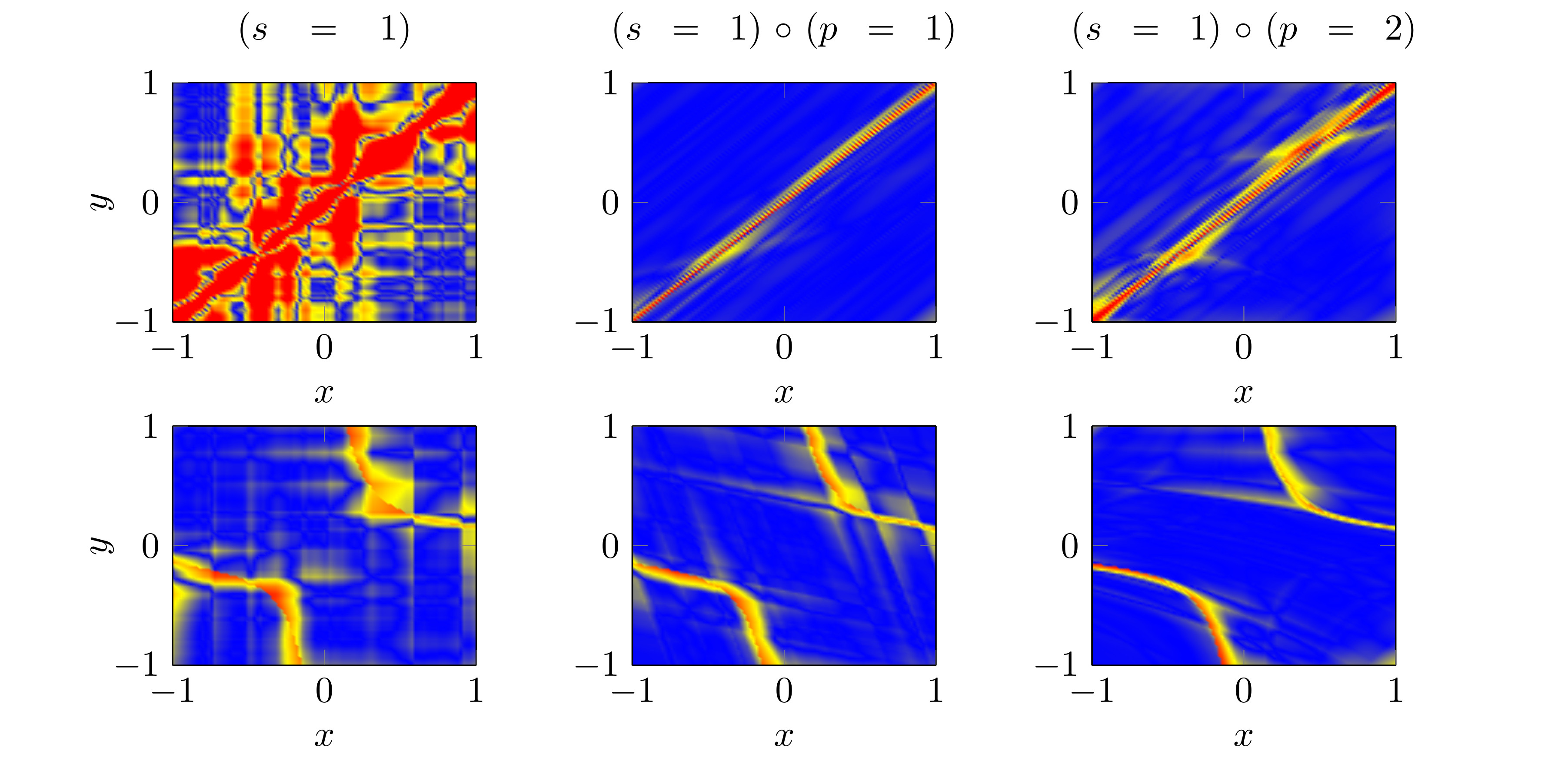}
\caption{The pointwise error for standard interpolation with $w = f^*_{X,Y}$ (left) and for concatenated interpolation with outer kernel $K_{\text{TensorMat\'{e}rn},1}$
and inner kernel $K_{\text{Poly},1}$ (mid) or $K_{\text{Poly},2}$ (right), respectively. We plotted both, the error for $h_1$ (top) and $h_2$ (bottom).
The color scale ranges from blue ($0 \%$ error) to red (more than $10 \%$ error), 
where the percentage has to be understood with respect to the $\| \cdot \|_{L_{\infty}}$ norm of $h_1$ or $h_2$, respectively.}
\label{fig:MatPolTPInt}
\end{figure}

\subsubsection{Regression}
Now we have a look at solving the least-squares regression problem \eqref{eq:regressionConcFiniteDimVec}. 
To determine the optimal 
parameters $\lambda$ and $\mu$, we run a $5$-fold cross-validation on the input data for all possible choices $\lambda, \mu \in \{ 2^{-2t + 1} \mid t = 1,\ldots,10\}$. 
Subsequently, we use the parameter pair $(\lambda, \mu)$ for which the smallest function value of \eqref{eq:regressionConcFiniteDimVec} is achieved and run the regression 
algorithm on the whole input data set to obtain our final results. We compare the two-layer case with the single-layer regression, see also \eqref{eq:regressStd},
with the parameter $\lambda$, which achieves the smallest error, i.e.~we compare to the best possible single-layer solution.

\begin{figure}
\centering
\includegraphics{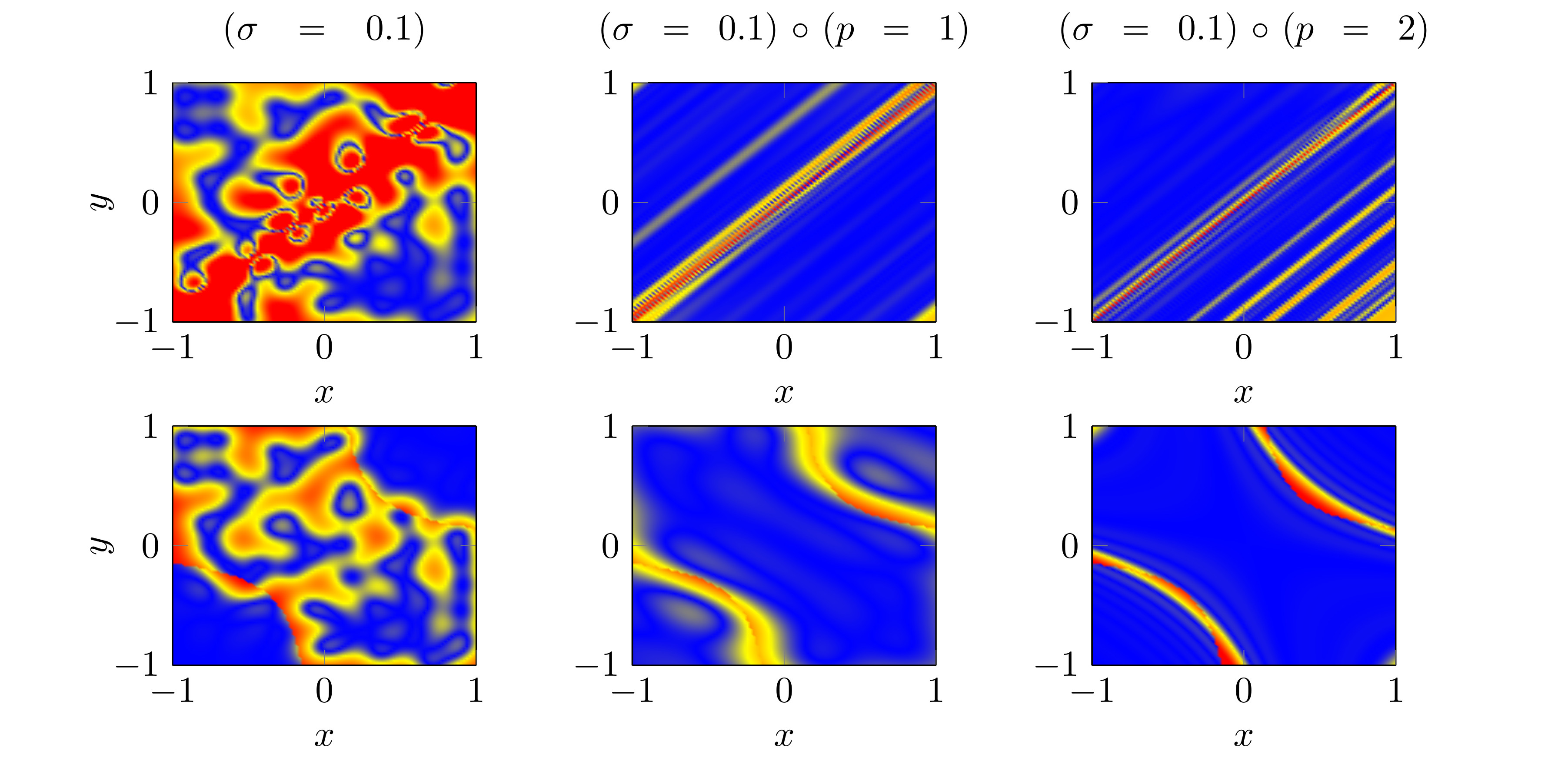}
\caption{The pointwise error for standard least-squares with $w = f^{\lambda}_{X,Y}$ (left) and for concatenated least-squares with outer kernel $K_{\text{Gauss},0.1}$ and 
inner kernel $K_{\text{Poly},1}$ (mid) or 
$K_{\text{Poly},2}$ (right), respectively. 
We plotted both, the error for $h_1$ (top) and $h_2$ (bottom). The color scale ranges from blue ($0 \%$ error) to red (more than $10 \%$ error), 
where the percentage has to be understood with respect to the $\| \cdot \|_{L_{\infty}}$ norm of $h_1$ or $h_2$, respectively.}
\label{fig:GauPolReg}
\end{figure}

Since the results for interpolation and least-squares regression with the same kernel choices as above
happen to be similar, we employ an outer kernel of Gaussian type $K = K_{\text{Gauss},\sigma}$ with $\sigma = 0.1$ instead of Mat\'{e}rn type here. 
For the inner kernel we again choose $K_{\mathcal{I}} = K_{\text{Poly},p}$ with $p=1,2$.
As we observe in Figure
\ref{fig:GauPolReg}, there is a significant improvement of the two-layer approach over the single-layer one.
Note that we deliberately employ the kernel width $\sigma = 0.1$, which appears to be too small for single-layer regression. 
However, the two-layer approach seems to remedy this bad choice automatically by adjusting the inner transformation accordingly. In this regard, the algorithm 
can also be understood as an implicit hyperparameter tuner.

\subsubsection{Linear outer kernel} \label{sec:linOuter}
In this section, we again want to emphasize the difference of our approach, which allows for nonlinear outer kernels, 
to the MKL-type RLS2 algorithm of \cite{DinuzzoPhd}, where only a linear outer kernel is considered and the inner kernel is given by a diagonal matrix with 
its entries being different scalar-valued (nonlinear) kernels. To this end, we run our two-layer least-squares regression approach for the following two settings:
\begin{itemize}
 \item[(1)] Outer polynomial kernel $K = K_{\text{Poly},1}$ of order $1$, inner mixture kernel $\boldsymbol{K}(\bx,\by)$,
 \item[(2)] Outer Mat\'{e}rn kernel $K = K_{\text{TensorMat\'{e}rn},1}$ of order $1$, inner mixture kernel $\boldsymbol{K}(\bx,\by)$.
\end{itemize}
For the inner mixture kernel, we deviate from \eqref{eq:innerKernelChoice} and from $D=2$ here. To this end, we set $D=5$ and use a diagonal kernel $\boldsymbol{K}$ with 
different scalar-valued kernels as entries. For the five scalar-valued kernels we choose three Gaussian kernels $K_{\text{Gauss},\sigma}$ with $\sigma = 0.1, 1, 10$ 
and two polynomial kernels $K_{\text{Poly},p}$ with $p=1,2$. 
Setting (1) serves to represent the RLS2 algorithm\footnote{Note however, that we did not use a diagonal scaling of the linear kernel 
and our optimization algorithm is different from the one used in \cite{DinuzzoPhd}, which is adjusted to the problem with a linear outer kernel.
}, where similar choices for the inner kernel have been made, see \cite{DinuzzoPhd}.
To determine the optimal parameters $\lambda, \mu \in \{ 10^{-2t + 1} \mid t = 1,\ldots,6\}$, we again run a $5$-fold crossvalidation\footnote{Note that we scan 
a coarser (but wider) range than in the previous section, which seemed to be appropriate here.}.
The results can be found in figure \ref{fig:DinuzzoComparison}. As we have already seen for interpolation, the structure of the function $h_1$
admits a good representation by a two-layer kernel discretization of type (2). However, despite the quite generic choice of the inner kernel in setting (1), the two-layer 
kernel approach with a linear outer kernel is not able to find a good representation of the function. 
This shows that a nonlinear choice for the outer kernel can be necessary to find suitable approximations by the two-layer algorithm.
Although the results do not differ that much for $h_2$, we again see that there is a slight advantage in approximating with a nonlinear outer kernel.

\begin{figure}
\centering
\includegraphics{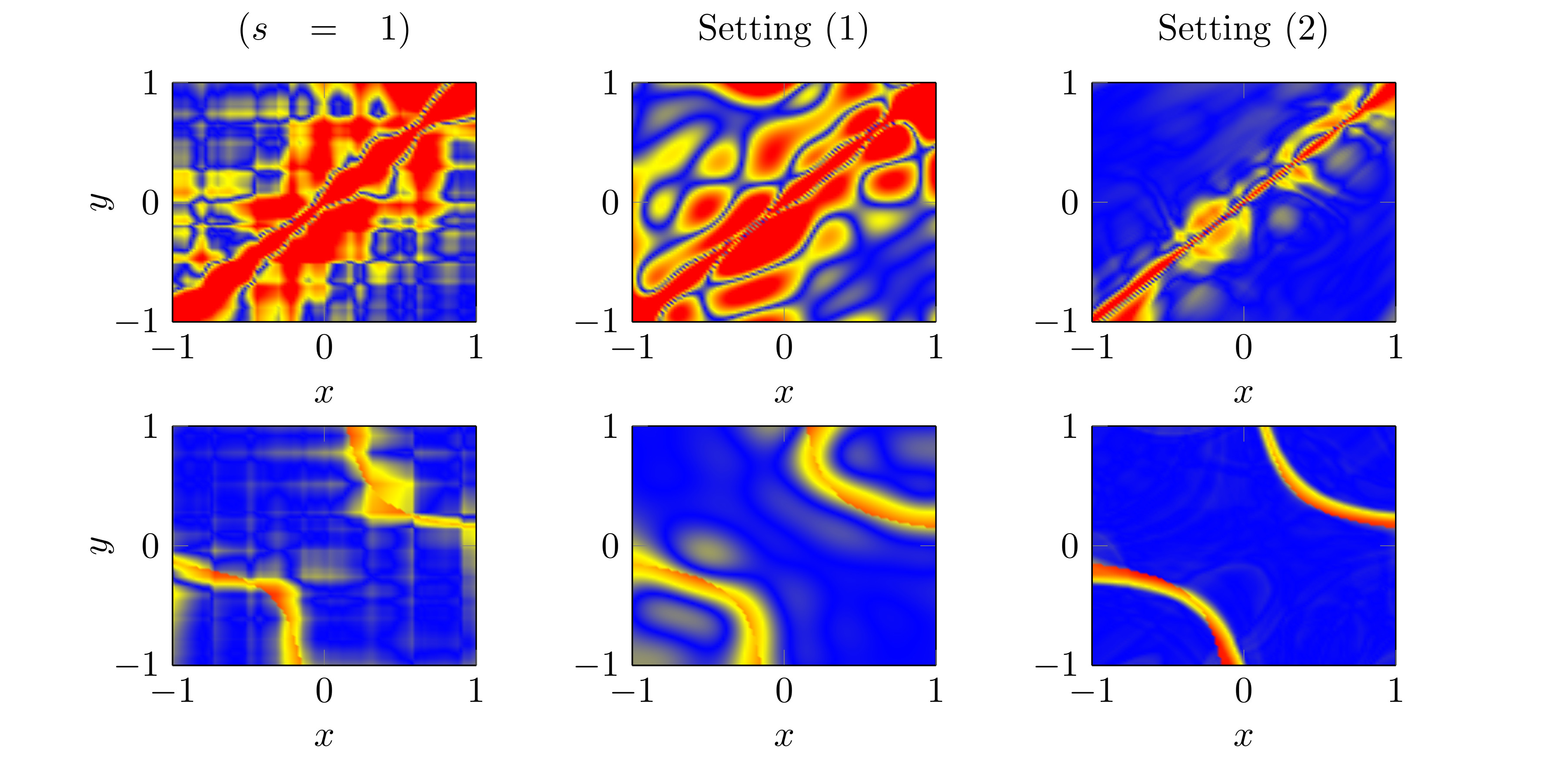}
\caption{The pointwise error for standard least-squares with Mat\'{e}rn kernel and $w = f^{\lambda}_{X,Y}$ (left) and for 
concatenated least-squares with setting (1) (mid) and setting (2) (right) from section \ref{sec:linOuter}.
We plotted both, the error for $h_1$ (top) and $h_2$ (bottom). The color scale ranges from blue ($0 \%$ error) to red (more than $10 \%$ error), 
where the percentage has to be understood with respect to the $\| \cdot \|_{L_{\infty}}$ norm of $h_1$ or $h_2$, respectively.}
\label{fig:DinuzzoComparison} 
\end{figure}

\subsection{Transformation by the inner function}

To get a better impression on how the two-layer algorithms work, we exemplarily inspect the inner function $\boldsymbol{g}$ in the case of interpolation with 
$K = K_{\text{TensorMat\'{e}rn},s}$ for $s=1$ and 
$K_{\mathcal{I}} = K_{\text{Poly},p}$ for $p=1$ or $p=2$, i.e.~for the setting from Section \ref{subsec:interpolatwolayerexample}. 
To this end, we depict isotropic grid points in $\Omega = [-1,1]^2$ and have a look at how these 
points are transformed by $\boldsymbol{g}$ in Figure \ref{fig:deformation}. 

\begin{figure} 
\centering
\includegraphics[width=0.99\columnwidth]{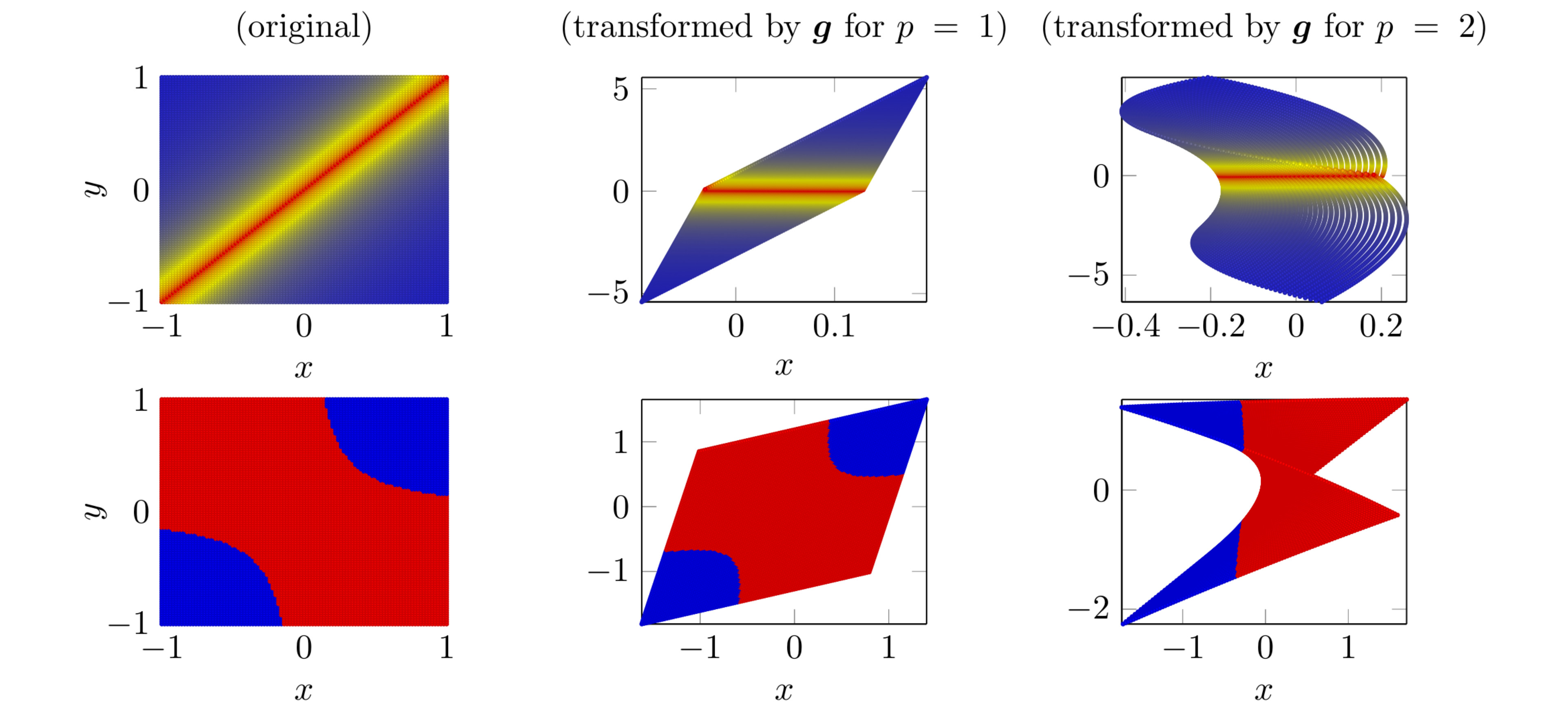}
\caption{The transformation of the isotropic grid points (left) by the inner function with $p=1$ (mid) and $p=2$ (right). The underlying problem is interpolation of $h_1$ (top) 
and $h_2$ (bottom) for the outer kernel $K_{\text{TensorMat\'{e}rn},1}$ and the inner kernel $K_{\text{Poly},p}$. The color scale represents the values of $h_1$ or $h_2$, respectively.}
\label{fig:deformation}
\end{figure}

We observe that for $h_1$ in both cases $p=1$ and $p=2$, the inner function aligns the kink almost perpendicular to the $y$-axis.
Therefore, one can easily characterize the kink by the $y$-coordinate after the inner transformation. 
This reduces the original two-dimensional kink description $x-y=0$ to just the one-dimensional description $y=0$. While the function with the 
kink along the diagonal does not reside in the tensor-product Mat\'{e}rn space of order $1$, which corresponds to the outer kernel in this example, 
a function with a kink parallel to one of the coordinate axes does. 
Therefore, the inner function $\boldsymbol{g}$ transforms the domain in such a way that the result resides in the RKHS to which 
the outer function belongs.

Considering the test function $h_2$, we see that a linear inner transformation, i.e.~$p=1$, essentially just rotates and shears the domain
and does not change the alignment of the jump very much. 
However, in the case $p=2$, the inner function $\boldsymbol{g}$ manages to transform the domain in such a way that the jump is now almost parallel to the $y$-axis. 
We observe that the pointwise errors in Figure \ref{fig:MatPolTPInt} really benefit from this transformation and the jump is resolved almost perfectly. 
Overall, we see that the inner function $\boldsymbol{g}$ tries to align the features of the original test function in such a way that they can be easily resolved by the outer function 
$f$.

\section{Conclusion}
\label{sec:Conclusion}

In this paper, we presented both a finite- and an infinite-sample representer theorem for concatenated machine learning problems.
In the finite-sample case, the statement essentially boils down to the fact that the a priori infinite-dimensional optimization problem, which appears when dealing with
function compositions from reproducing kernel Hilbert spaces, can be recast into 
a finite-dimensional optimization problem, where we only have to deal with at most $N$ kernel translates in each layer of the composition. Here, $N$ denotes the number of input data points.
In the infinite-sample case, we derived an analogous result stating that the solution in each layer is an element of the image space of the integral operator 
defined by the corresponding kernel evaluated at the innermost functions.
We introduced a simple neural network architecture, which represents the concatenated functions we are dealing with. Furthermore, we
established a connection between our representer theorem and two types of state-of-the-art deep learning algorithms, namely multi-layer multiple kernel learning and deep kernel networks.
Finally, we presented a detailed analysis on a two-layer interpolation and a two-layer least-squares regression algorithm, which can directly be derived from our representer theorem. 
We illustrated the operating principles of these algorithms with the help of two artificial test functions and explained why the two-layer approach is 
able to remedy the shortcomings of a single-layer variant. Furthermore, we highlighted that the use of a nonlinear outer kernel, instead of a linear one as in 
\cite{DinuzzoPhd}, can be inevitable to obtain good two-layer approximations.
Nevertheless, the nonlinearity of the outer layer makes the numerical treatment of the underlying optimization problem more difficult.

While we presented specific two-layer ($L=2$) algorithms and applied them to two-dimensional ($d=2$) toy problems for illustrative reasons, our representer theorems 
can also be applied in the high-dimensional case with an arbitrary number of layers. Note furthermore that, apart from interpolation and least-squares regression, also
more general choices of the loss function ${\cal L}$ and the regularizers $\Theta_l$ are allowed in \eqref{eq:generalMinFunct}. Therefore, one can also 
think of multi-layer support vector machines for instance. The construction of such efficient deep kernel learning algorithms for high-dimensional problems and a
thorough analysis of the interplay between the number of layers $L$ and the dimension $d$ will be future work.

\appendix
\section*{Appendix A: Remainder of the proof of theorem \ref{theo:RepTheoInfiniteSample}} \label{appendix}

To continue the proof of theorem \ref{theo:RepTheoInfiniteSample}, we note that 
we already showed that $f_1$ has the desired structure \eqref{eq:RepresenterIntegral}.
Let us assume we have shown \eqref{eq:RepresenterIntegral} for all $f_1,\ldots,f_{l-1}$ for an $l \in \{2,\ldots,L\}$. To obtain \eqref{eq:RepresenterIntegral} 
for $f_l$, we proceed in the same fashion as in the first part of the proof in section \ref{sec:InfRepTheo}. To this end, we now define
$J_{g_{l+1},\ldots,g_L} : {\cal H}_l \to [0,\infty)$ by
$$
J_{g_{l+1},\ldots,g_L}(g_l) := \int_{R_{l+1} \times R_1} \tilde{\cal L}_l \left(y, g_l(\bxi)\right) \ \mathrm{d} G_{l,\star}(\PP)(\bxi,y) + \lambda_l \| g_l \|_{{\cal H}_l}^2,
$$
where $\tilde{\cal L}_l(y,\bz) := {\cal L}(y,f_1 \circ \ldots \circ f_{l-1}(\bz))$
and $G_{l,\star}(\PP)$ is the pushforward of $\PP$ onto $R_{l+1} \times R_1$ defined by $G_l(\bx,y) = (g_{l+1} \circ \ldots \circ g_L(\bx), y)$.
Then it holds
\begin{align*}
& \ \min_{g_l \in {\cal H}_l,\ldots,g_L \in {\cal H}_L} J(f_1,f_2,\ldots,f_{l-1},g_l,g_{l+1}, \ldots ,g_L) \\ 
 = & \ \min_{g_{l+1} \in {\cal H}_{l+1},\ldots,g_L \in {\cal H}_L}
\left( \min_{g_l \in {\cal H}_l} 
J_{g_{l+1},\ldots,g_{l+1}}(g_l) \right) + \sum_{i=l+1}^L \lambda_{i} \| g_i \|^2_{{\cal H}_i}
\end{align*}
and we need to show that a minimizer of $J_{g_{l+1},\ldots,g_L}$ admits a representation of type \eqref{eq:RepresenterIntegral}.
To this end, we begin by defining a Nemitski vector loss function and we subsequently prove that these loss functions admit the representation we need.

\begin{definition} \label{def:NemVec} 
 Let ${\cal L}: R_1 \times D \to [0,\infty)$ for some domain $D \subset \R^d$.
 Let $\PP_{R_1}$ denote the marginal distribution of $\PP$ w.r.t.\ the second variable.
 We call ${\cal L}$ a $\PP$-integrable Nemitski vector loss,
 if there exist $b: R_1 \to [0,\infty)$ with $b \in L_{1,\PP_{R_1}}(R_1)$ and a measurable, increasing $h :[0,\infty) \to [0,\infty)$ such that
\begin{align*}
{\cal L}(y,\bz) \leq b(y) + h(\|\bz\|) & ~ \text{ for all } (y,\bz) \in R_1 \times D.
\end{align*} 
If ${\cal L}$ is $k$-times differentiable w.r.t.\ the second variable for all $y \in R_1$, we call it a $k$-times differentiable Nemitski vector loss.
\end{definition}

\begin{lemma} \label{lemma:repLemma}
 Let $l \in \{2,\ldots,L\}$ and let $\PP^l$ be a distribution\footnote{Note that we set $R_{L+1} := D_L = \Omega$.} on $R_{l+1} \times R_1$ and
 let ${\cal L}^{\star}$ be a $\PP^l$-integrable and $1$-differentiable Nemitski vector loss on $R_1 \times R_{l}$ such that 
 the derivative w.r.t.\ the second argument $\nabla_2 {\cal L}$ fulfills
\begin{align*}
 \left\| \nabla_2 {\cal L}^{\star}(y, \bz) \right\| \leq b^{\star}(y) + h^{\star}(\|\bz\|) & ~ \text{ for all } (y,\bz) \in R_1 \times R_{l}
\end{align*} 
for some $b^{\star} \in L_{1,\PP^l_{R_1}}(R_1)$ and a measurable, increasing $h^{\star} :[0,\infty) \to [0,\infty)$. Then, the functional ${\cal R}_{l,\PP^{l}}: {\cal H}_l \to [0,\infty)$ 
defined by
\begin{equation*} 
{\cal R}_{l,\PP^{l}}(f) := \int_{R_{l+1} \times R_1} {\cal L}^{\star}(y,f(\bz)) \ \mathrm{d} \PP^l(\bz,y)
\end{equation*}
is Frechet differentiable and the derivative $\mathrm{d} {\cal R}_{l,\PP^{l}} : {\cal H}_l \to {\cal B}({\cal H}_l,\R)$  is given by
\begin{equation} \label{eq:riskdiff}
\mathrm{d} {\cal R}_{l,\PP^{l}}(f)(g) = \int_{R_{l+1} \times R_1} \nabla_2 {\cal L}^{\star}(y,f(\bz))^T \cdot g(\bz) \ \mathrm{d} \PP^l(\bz,y).
\end{equation}
Furthermore, a critical point of $J^{\star} : {\cal H}_l \to [0,\infty)$ defined by
$$
J^{\star}(f) := {\cal R}_{l,\PP^{l}}(f) + \lambda_l \| f \|_{{\cal H}_l}^2
$$
is given by 
\begin{equation} \label{eq:critPoint}
 f(\cdot) = - \frac{1}{2 \lambda_l} \int_{R_{l+1} \times R_1} K_l(\cdot,\bz) \cdot \nabla_2 {\cal L}^{\star}(y,f(\bz)) \ \mathrm{d} \PP^l(\bz,y).
\end{equation}
\end{lemma}
\begin{proof}
 We have
 \begin{align*}
 & \ \lim_{\|g\|_{{\cal H}_l} \to 0} \frac{{\cal R}_{l,\PP^{l}}(f+g) - {\cal R}_{l,\PP^{l}}(f) - \int_{R_{l+1} \times R_1} \nabla_2 {\cal L}^{\star}(y,f(\bz))^T \cdot g(\bz) \ \mathrm{d} \PP^l(\bz,y)}{\| g \|_{{\cal H}_l}}
 \\ = & \
  \lim_{\|g\|_{{\cal H}_l} \to 0} \int_{R_{l+1} \times R_1} \frac{{\cal L}^{\star}(y,f(\bz) + g(\bz)) - {\cal L}^{\star}(y,f(\bz)) - \nabla_2 {\cal L}^{\star}(y,f(\bz))^T \cdot g(\bz) }{\| g \|_{{\cal H}_l}}\ \mathrm{d} \PP^l(\bz,y)
  \\ \stackrel{(*)}{=} & \
   \int_{R_{l+1} \times R_1} \lim_{\|g\|_{{\cal H}_l} \to 0} \frac{{\cal L}^{\star}(y,f(\bz) + g(\bz)) - {\cal L}^{\star}(y,f(\bz)) - \nabla_2 {\cal L}^{\star}(y,f(\bz))^T \cdot g(\bz) }{\| g \|_{{\cal H}_l}}\ \mathrm{d} \PP^l(\bz,y)
   = 0,
 \end{align*}
 where the last equation follows from the differentiability of ${\cal L}^{\star}$ and $(*)$ follows by the dominated convergence theorem since the integrand is bounded by
 \begin{align*}
  & \ \left| \frac{{\cal L}^{\star}(y,f(\bz) + g(\bz)) - {\cal L}^{\star}(y,f(\bz)) - \nabla_2 {\cal L}^{\star}(y,f(\bz))^T \cdot g(\bz) }{\| g \|_{{\cal H}_l}} \right| \\
  = & \ \left| \frac{\nabla_2 {\cal L}^{\star}(y,cf(\bz) + (1-c)g(\bz))^T \cdot g(\bz) - \nabla_2 {\cal L}^{\star}(y,f(\bz))^T \cdot g(\bz) }{\| g \|_{{\cal H}_l}} \right| \\
  \leq & \ 2 b^{\star}(y) + h^{\star}(\|cf(\bz) + (1-c)g(\bz)\|) + h^{\star}(\|f(\bz)\|) \\
 \end{align*}
for some $c \in [0,1]$ due to the mean value theorem. Since the last line is bounded by 
$2 b^{\star}(y) + 2 h^{\star}(\|f(\bz)\| + 1)$ independently of $g$ for any $g$ with $\| g \|_{{\cal H}_l} \leq 1$, the dominated convergence theorem can be applied, which
proves \eqref{eq:riskdiff}. 

Since a critical point $f$ of $J^{\star}$ fulfills
$$
0 = \mathrm{d} J^{\star}(f)(g) = \mathrm{d} {\cal R}_{l,\PP^{l}}(f)(g) + 2 \lambda_l \langle f, g \rangle_{{\cal H}_l},
$$
for all $g \in {\cal H}_l$, we obtain
\begin{align*}
\langle f, g \rangle_{{\cal H}_l}  = & \ - \frac{1}{2\lambda_l} \mathrm{d} {\cal R}_{l,\PP^{l}}(f)(g) \\
 = & \ - \frac{1}{2\lambda_l} \int_{R_{l+1} \times R_1} \nabla_2 {\cal L}^{\star}(y,f(\bz))^T \cdot g(\bz) \ \mathrm{d} \PP^l(\bz,y) \\
 = & \ - \frac{1}{2\lambda_l} \int_{R_{l+1} \times R_1} \nabla_2 {\cal L}^{\star}(y,f(\bz))^T \cdot \sum_{i=1}^{d_l} \langle g, K_l(\cdot, \bz) \boldsymbol{e}_i \rangle_{{\cal H}_l} \cdot \boldsymbol{e}_i \ \mathrm{d} \PP^l(\bz,y) \\
 = & \ - \frac{1}{2\lambda_l} \sum_{i=1}^{d_l} \left\langle \int_{R_{l+1} \times R_1} \nabla_2 {\cal L}^{\star}(y,f(\bz))^T K_l(\cdot, \bz) \boldsymbol{e}_i 
\ \mathrm{d} \PP^l(\bz,y), g \right\rangle_{{\cal H}_l} \cdot \boldsymbol{e}_i 
\end{align*}
with the reproducing property of $K_l$, which is equivalent to the Bochner-type integral formulation \eqref{eq:critPoint}. This finishes the proof.
\end{proof}

Now, we can apply lemma \ref{lemma:repLemma} with $\PP^l = G_{l,\star}(\PP)$ and ${\cal L}^{\star} = \tilde{\cal L}_l$, which shows that 
a critical point $g^{\star}_l$ of $J_{g_{l+1},\ldots,g_L}$ can be written as 
\begin{align} \label{eq:critPointFinal}
g^{\star}_l(\cdot) = & \ - \frac{1}{2 \lambda_l} \int_{R_{l+1} \times R_1} K_l(\cdot,\bxi) \cdot \nabla_2 \tilde{\cal L}_l(y,g^{\star}_l(\bxi)) \ \mathrm{d} G_{l,\star}(\PP)(\bxi,y) \\
= & \ - \frac{1}{2 \lambda_l} \int_{\Omega \times R_1} K_l(\cdot,g_{l+1} \circ \ldots \circ g_L(\bx)) \cdot \nabla_2 \tilde{\cal L}_l(y,g^{\star}_l \circ g_{l+1} \circ \ldots \circ g_L(\bx))
\ \mathrm{d} \PP(\bx,y), \nonumber
\end{align}
which is of type \eqref{eq:RepresenterIntegral}. Therefore, it just remains to show that $\tilde{\cal L}_l$ fulfills 
the prerequisites of lemma \ref{lemma:repLemma} and that $A_{f_l,f_{l+1},\ldots,f_{L}}(\bx,y) := 
\nabla_2 \tilde{\cal L}_l(y,f_l \circ f_{l+1} \circ \ldots \circ f_L(\bx)) \in L_{1,\PP}$.

\begin{lemma} \label{lemma:IsNemitski}
  $\tilde{{\cal L}}_l$ is a $G_{l,\star}(\PP)$-integrable and $1$-differentiable Nemitski loss
 and the derivative w.r.t.\ the second argument fulfills
\begin{align}  \label{eq:NemitskiNormBound}
 \left\| \nabla_2 \tilde{\cal L}_l(y, \bz) \right\| \leq \tilde{b}(y) + \tilde{h}(\|\bz\|) & ~ \text{ for all } (y,\bz) \in R_1 \times R_l
\end{align} 
for a $\tilde{b} \in L_{1,G_{l,\star}(\PP)_{R_1}}(R_1)$ and a measurable, increasing $\tilde{h} :[0,\infty) \to [0,\infty)$.
\end{lemma}
\begin{proof}
Since 
\begin{align*}
|\tilde{\cal L}_l(y,\bz)| = & \ |{\cal L}(y,f_1 \circ \ldots \circ f_{l-1}(\bz))|
\leq b_0(y) + h_0(f_1 \circ \ldots \circ f_{l-1}(\bz)) \leq   b_0(y) + h_0(\| f_1 \|_{\infty}),
\end{align*}
$\tilde{\cal L}_l$ is a $G_{l,\star}(\PP)$-integrable Nemitski-loss. Here, we again used that $f_1 \in {\cal H}_1 \hookrightarrow C(D_1)$ because 
of \eqref{eq:Kernelprereq}. Since $f_1, \ldots, f_{l-1}$ are differentiable because the respective kernels are in $C^1$, $\tilde{\cal L}_l$ is also 
$1$-differentiable by the chain rule. It remains to show \eqref{eq:NemitskiNormBound}. To this end, note that the chain rule gives us
\begin{align*}
& \ \nabla_2 \tilde{\cal L}_l(y, \bz)(\cdot) = \frac{\partial}{\partial \bz} \left({\cal L}\left(y,f_1 \circ \ldots \circ f_{l-1}(\bz)\right) \right) \\
= & \ {\cal L}^{(1)}(y,f_1 \circ \ldots \circ f_{l-1}(\bz)) \cdot \mathrm{d}f_1(f_2 \circ \ldots \circ f_{l-1}(\bz)) \left(\mathrm{d}f_2(f_3 \circ \ldots \circ f_{l-1}(\bz))
\left(\ldots \mathrm{d}f_{l-1}(\bz)(\cdot) \right)\right),
\end{align*}
which leads to
\begin{align} \label{eq:specificNemitskiBound}
\| \nabla_2 \tilde{\cal L}_l(y, \bz) \| \leq & \ \left( b_1(y) + h_1(|f_1 \circ \ldots \circ f_{l-1}(\bz)|) \right) \cdot \prod_{i=1}^{l-1} \sup_{\bx_i \in D_i} \| \mathrm{d} f_i(\bx_i) \|_{{\cal B}(D_i,R_i)} \nonumber \\
\leq & \ \left( b_1(y) + h_1(\| f_1 \|_{\infty})\right) \cdot \prod_{i=1}^{l-1} \sup_{\bx_i \in D_i} \| \mathrm{d} f_i(\bx_i) \|_{{\cal B}(D_i,R_i)}.
\end{align}
Because of our assumption that we already showed \eqref{eq:RepresenterIntegral} for $f_1,\ldots,f_{l-1}$ and because of \eqref{eq:Kernelprereq}, we 
get by the dominated convergence theorem that 
$$
 \sup_{\bx_i \in D_i} \| \mathrm{d} f_i(\bx_i) \|_{{\cal B}(D_i,R_i)} \leq \frac{1}{2\lambda_i} c_i \| A_{f_i,\ldots,f_L} \|_{L_{1,\PP}}
$$
for all $i=1,\ldots,l-1$. 
Therefore, by setting $\tilde{b}(y) := c \cdot b_1(y)$ and choosing a constant $\tilde{h} := c \cdot  h_1(\| f_1 \|_{\infty})$ with 
$c := \prod_{i=1}^{l-1} \frac{1}{2\lambda_i} c_i \| A_{f_i,\ldots,f_L} \|_{L_{1,\PP}} < \infty$, we obtain \eqref{eq:NemitskiNormBound}.
\end{proof}

Applying lemma \ref{lemma:repLemma} and lemma \ref{lemma:IsNemitski} shows us that $f_l$ fulfills the integral equation \eqref{eq:critPointFinal}. 
To conclude the proof of theorem \ref{theo:RepTheoInfiniteSample}, we note that 
$$
A_{f_l,f_{l+1},\ldots,f_{L}}(\bx,y) := \nabla_2 \tilde{\cal L}_l(y,f_l \circ f_{l+1} \circ \ldots \circ f_L(\bx)) \in L_{1,\PP},
$$
which directly follows from \eqref{eq:specificNemitskiBound} and the fact that $b_1 \in L_{1,\PP_{R_1}}$. This finally shows that $f_l$ admits a representation of 
type \eqref{eq:RepresenterIntegral}. Since the argument is valid for each $l=2,\ldots,L$ and we already proved \eqref{eq:RepresenterIntegral} for $l=1$ in section \ref{sec:InfRepTheo}, 
this finishes the proof of theorem \ref{theo:RepTheoInfiniteSample}.


\bibliographystyle{amsplain}
\bibliography{bibliography}
\end{document}